\documentclass{article}

\usepackage{enumitem}
\usepackage{microtype}
\usepackage{graphicx}
\usepackage{booktabs} 

\usepackage[algo2e,linesnumbered,lined,ruled,noend]{algorithm2e}
\usepackage{amsmath,amssymb, latexsym, amsthm}
\usepackage{makecell}
\usepackage{subcaption}
\usepackage{adjustbox}
\usepackage{sidecap}
\usepackage{mathtools}

\usepackage{hyperref}

\usepackage[accepted]{icml2019}

\icmltitlerunning{Provable Approximations for Constrained $\ell_p$ Regression}

\newcommand{\cost}{\mathrm{cost}}

\newcommand{\sign}{\mathrm{sign}}

\newcommand{\sphere}{\mathbb{S}}
\newcommand{\M}{\mathrm{m}}
\newcommand{\F}{\mathbb{F}}

\newcommand{\eps}{\varepsilon}
\newcommand{\dist}{\mathrm{dist}}

\newcommand{\range}{\mathrm{range}}
\newcommand{\lip}{\mathrm{lip}}
\newcommand{\smallest}{\mathrm{small}}

\newcommand{\opt}{ \mathrm{opt}}
\newcommand{\perms}{\mathrm{Perms}}
\newcommand{\REAL}{\mathbb{R}}
\newcommand{\br}[1]{\left\{#1\right\}}
\providecommand{\norm}[1]{\left\lVert#1\right\rVert}

\DeclarePairedDelimiter{\floor}{\lfloor}{\rfloor}
\DeclareMathOperator*{\argmin}{arg\,min}
\DeclareMathOperator*{\argmax}{arg\,max}

\newcommand{\algnameXCandidates}{\textsc{Calc-x-candidates}}
\newcommand{\algCalcOpt}{\textsc{Calc-opt}}

\newcommand{\matchingalgname}{\textsc{Match-Algorithm}}

\newcommand{\OurAlg}{\textsc{Our-Algorithm}}
\newcommand{\OurMatchingAlg}{\textsc{Our-Match-Alg}}
\newcommand{\Maple}{\textsc{Maple30}}
\newcommand{\MaplePro}{\textsc{Maple100}}
\newcommand{\MapleRansac}{\textsc{Maple-Ransac}}
\newcommand{\MapleICP}{\textsc{Maple30-ICP}}
\newcommand{\MapleProICP}{\textsc{Maple100-ICP}}
\newcommand{\Bnb}{\textsc{Yalmip-bmibnb}}
\newcommand{\BnbRansac}{\textsc{bmibnb-Ransac}}
\newcommand{\BnbICP}{\textsc{bmibnb-ICP}}
\newcommand{\Baron}{\textsc{Yalmip-baron}}
\newcommand{\BaronRansac}{\textsc{Baron-Ransac}}
\newcommand{\BaronICP}{\textsc{Baron-ICP}}

\newtheorem{theorem}{Theorem}[section]
\newtheorem{corollary}[theorem]{Corollary}
\newtheorem{lemma}[theorem]{Lemma}
\newtheorem{observation}[theorem]{Observation}

\newtheorem{definition}[theorem]{Definition}

\renewcommand{\paragraph}[1]{\medskip\noindent\textbf{{#1}}}

\begin{document}
\pagenumbering{arabic}
\pagestyle{plain}
\twocolumn[
\icmltitle{Provable Approximations for Constrained $\ell_p$ Regression}



\icmlsetsymbol{equal}{*}

\begin{icmlauthorlist}
\icmlauthor{Ibrahim Jubran}{uni}
\icmlauthor{David Cohn}{uni}
\icmlauthor{Dan Feldman}{uni}
\end{icmlauthorlist}

\icmlaffiliation{uni}{Robotics and Big Data Lab, Department of Computer Science, University of Haifa, Haifa, Israel}

\icmlcorrespondingauthor{Ibrahim Jubran}{ibrahim.jub@gmail.com}
\icmlcorrespondingauthor{David Cohn}{david.cohn.82@gmail.com}

\icmlkeywords{Machine Learning, Non-Convex Optimization, Linear Regression, Approximation Algorithms}

\vskip 0.3in
]



\printAffiliationsAndNotice{}  

\begin{abstract}
The $\ell_p$ linear regression problem is to minimize $f(x)=||Ax-b||_p$ over $x\in\mathbb{R}^d$, where $A\in\mathbb{R}^{n\times d}$, $b\in \mathbb{R}^n$, and $p>0$.
To avoid overfitting and bound $||x||_2$, the \emph{constrained $\ell_p$ regression} minimizes $f(x)$ over every \emph{unit vector} $x\in\mathbb{R}^d$.
This makes the problem non-convex even for the simplest case $d=p=2$.
Instead, ridge regression is used to minimize the Lagrange form $f(x)+\lambda ||x||_2$ over $x\in\mathbb{R}^d$, which yields a convex problem in the price of calibrating the regularization parameter $\lambda>0$.

We provide the first provable constant factor approximation algorithm that solves the constrained $\ell_p$ regression directly, for every constant $p,d\geq 1$.
Using core-sets, its running time is $O(n \log n)$ including extensions for streaming and distributed (big) data.
In polynomial time, it can handle outliers, $p\in (0,1)$ and minimize $f(x)$ over every $x$ and permutation of rows in $A$.

Experimental results are also provided, including open source and comparison to existing software.
\end{abstract}

\section{Introduction} \label{sec:probState}
One of the fundamental problems in machine learning is $\ell_2$ linear regression, where the goal is to fit a hyperplane that minimizes the sum of squared vertical distances to a set of $n\gg d$ input $(d+1)$-dimensional points (samples, vectors, training data). Formally, the input is an $n\times d$ matrix $A = (a_1\mid \cdots\mid a_n)^T$ and a vector $b=(b_1,\cdots,b_n)^T$ in $\REAL^n$ that contains the $n$ labels (heights, or last dimension) of the points. The goal is to minimize the sum $\sum_{i=1}^n(a_i^Tx-b_i)^2$ over every $d$-dimensional vector $x$ of coefficients,
\begin{equation}\label{mainEq1}
\min_{x\in\REAL^d}\norm{Ax-b}_2.
\end{equation}

One disadvantage of these techniques is that overfitting may occur~\cite{buhlmann2011statistics}. For example, if the entries in $b$ are relatively small, then $x=(0,\cdots,0)$ may give an approximated but numerically unstable solution. Moreover,~\eqref{mainEq1} is not robust to outliers, in the sense that e.g. adding a row whose entries are relatively very large would completely corrupt the desired vector $x$.

This motivates the addition of a constraint $c>0$ on the norm of $x$, where $c$ is a constant that may depend on the scale of the input $A$ and $b$. Without loss of generality, we can assume $\norm{x}_2=c=1$, otherwise we divide entries ofP $b$ accordingly. The result is \emph{constrained $\ell_2$ regression problem},
\begin{equation}\label{mainEq2}
\min_{x\in \REAL^d:\norm{x}_2=1} \norm{Ax-b}_2.
\end{equation}
The special case $b=(0,\cdots,0)$ can be solved in $O(nd^2)$ time, where the optimum is the smallest singular value of $A$ and $x$ is the largest singular vector of $x$~\cite{golub1970singular}.

A generalization of~\eqref{mainEq2} for a given constant $p>0$ would be
\begin{equation}\label{mainEq3}
\min_{x\in \REAL^d, \norm{x}_2=1} \norm{Ax-b}_p,
\end{equation}
where $\norm{v}_p=\left(\sum_{i=1}^d |v_i|^p\right)^{1/p}$ for $v = (v_1,\cdots,v_d)\in\REAL^d$. Note that for $p<1$ we obtain a non-standard norm which is a non-convex function over $v\in R^d$.

Optimization problem~\eqref{mainEq3} for $p=1$ can be defined geometrically as follows.
Compute a point $x'$ on the unit sphere that minimizes the weighted sum of distances over $n$ given hyperplanes and $n$ multiplicative weights.
Here, the $i$th hyperplane is defined by its normal (unit vector) $a_i/\norm{a_i}$, its distance from the origin $b_i/\norm{a_i}$, and its weight $\norm{a_i}$. The weighted distance between $x'$ and the $i$th hyperplane is defined as $\norm{a_i}\cdot |\frac{a_i^T}{\norm{a_i}}x' -\frac{b_i}{\norm{a_i}}| = |a_i^Tx-b_i|$.

In the context of machine learning, in linear regression we wish to fit a hyperplane whose unit normal is $x'$, that minimizes the sum of squared vertical distances between the hyperplane at point $a_i$ (predicted value) and $b_i$ (the actual value), over every $i\in [n]$. In low-rank approximation (such as SVD / PCA) we wish to fit a hyperplane that \emph{passes through the origin} and whose unit normal is $x'$, that minimizes the sum of squared Euclidean distances between the data points and the hyperplane.
Our problem is a mixture of these two problems: compute a hyperplane that passes through the origin (as in low-rank approximation) and minimizes sum of squared vertical distances (as in linear regression).

Further generalization of~\eqref{mainEq3} suggests handling data with outliers. For example when one of the rows of $A$ is very noisy, or if an entry of $b$ is unknown. Let $k < n$ be the number of such outliers in our data. In this case, we wish to ignore the largest $k$ distances (fitting errors), i.e., consider only the closest $s=n-k$ points to $x$. Formally,
\begin{equation} \label{mainEq5}
\min_{x\in \REAL^d:\norm{x}_2=1}\norm{\smallest(Ax-b,n-k)}_p,
\end{equation}
where $\smallest(v,s) \in \REAL^s$ is a vector that consists of the smallest $s$ entries in $v\in\REAL^n$, where $s\in[0,n]$ is an integer.

In some cases, our set of observations is unordered, i.e., we do not know which observation in $b$ matches each point in $A$.
For example, when half of the points should be assigned to class $b_1=..=b_{n/2}=0$ and half of the points to class $b_{n/2+1}..=b_n=1$. Here, $\M:\br{1,\cdots,n}\to\br{1,\cdots,n}$ denotes a bijective function, called a \emph{matching function}, and $b_{\M} = (b_{\M(1)},\cdots,b_{\M(n)})^T$ denotes the permutation of the entries in $b$ with respect to $\M$. In this case, we need to compute
\begin{equation}\label{mainEq4}
\min_{x,\M} \norm{Ax-b_{\M}}_p,
\end{equation}
where the minimum is over every unit vector $x$ and matching function $\M:[n]\to [n]$.


\section{Related Work} \label{sec:related}
Regression problems are fundamental in statistical data analysis and have numerous applications
in applied mathematics, data mining, and machine learning; see references in~\cite{friedman2001elements,chatterjee2015regression}. Computing the simple (unconstrained) Linear regression in~\eqref{mainEq1} for the case $p=2$ was known already in the beginning of the previous century~\cite{pearson1905general}. Since the $\ell_p$ norm is a convex function, for $p=1$ it can be solved using linear programming, and in general for $p\geq 1$ using convex optimization techniques in time $(nd)^{O(1)}$ or using recent coreset (data summarization) technique~\cite{dasgupta2009sampling} in near-linear time.

To avoid overfitting and noise, there is a need to bound the norm of the solution $x$, which yields problem~\eqref{mainEq2} when $p=2$.
The constraint $\norm{x}=1$ in~\eqref{mainEq2} can be replaced by adding a Lagrange multiplier~\cite{rockafellar1993lagrange} $\lambda>0$ to obtain,
\begin{equation}\label{lag1}
\min_{\lambda\in\REAL,x\in\REAL^d} \norm{Ax-b}_2+\lambda\norm{x}_2-\lambda.
\end{equation}

Unfortunately,~\eqref{mainEq2} and~\eqref{lag1} are non-convex problems in quadratic programming as explained in~\cite{park2017general}, which are also NP-hard if $d=n$, so there is no hope for running time that is polynomial in $d$; see Conclusion section. It was proved in~\cite{jubran2018minimizing} that the problem is non-convex even if every input point (row) $a_i$ is on the unit circle and $d=2$.
Similarly, when we are allowed to ignore $k$ outlier as defined in~\eqref{mainEq5}, or if we use $M$-estimators, the problem is no longer convex.

Instead, a common leeway is to ``guess" the value of $\lambda$ in~\eqref{lag1}, i.e., turn it into an input parameter that is calibrated by the user and is called \emph{regularization term}~\cite{zou2005regularization} to obtain a relaxed convex version of the problem,
\[
\min_{x\in\REAL^d} \norm{Ax-b}_2+\lambda\norm{x}_2.
\]
This problem is the \emph{ridge regression} which is also called \emph{Tikhonov regularization} in statistics~\cite{hoerl1970ridge}, \emph{weight decay} in machine learning~\cite{krogh1992simple}, and \emph{constrained linear inversion method}~\cite{twomey1975comparison} in optimization. Many heuristics were suggested to calibrate $\lambda$ automatically in order to remove it such as automatic plug-in estimation, cross-validation, information
criteria optimization, or Markov chain Monte Carlo (MCMC)~\cite{kohavi1995study,gilks1995markov} but no provable approximations for the constrained $\ell_p$ regression~\eqref{mainEq2} are known; see~\cite{karabatsos2018marginal,karabatsos2014fast,cule2012semi} and references therein.

Another reminiscent approach is LASSO (least absolute shrinkage and selection operator)~\cite{tibshirani1996regression}, which replaces the non-convex constraint $\norm{x}_2=1$ in~\eqref{mainEq2} with its convex $\ell_1$ inequality $\norm{x}_1\leq t$ to obtain $\min_{x\in \REAL^d:\norm{x}_1\leq t} \norm{Ax-b}_2$ for some given parameter $t>0$.

LASSO is most common technique in regression analysis~\cite{zou2005regularization} for e.g. variable selection and compressed sensing~\cite{angelosante2009compressed} to obtain sparse solutions. As explained in~\cite{tibshirani1996regression,tibshirani1997lasso} these optimization problems can be easily extended to a wide variety of statistical models including generalized linear models, generalized estimating equations, proportional hazards models, and M-estimators. 

Alternatively, the $\ell_2$-norm in~\eqref{mainEq1} may be replaced by the $\ell_p$-norm for $p\geq 1$ to obtain the (non-constrained) \emph{$\ell_p$ regression}  \begin{equation}\label{mainEq33}
\min_{x\in \REAL^d} \norm{Ax-b}_p,
\end{equation}
which is convex for the case $p\geq 1$. Using $p\in(0,1]$ in~\eqref{mainEq33} is especially useful for handling outliers~\cite{ding2017l1} which arise in real-world data. However, for the (non-standard) $\ell_p$-norm, where $p<1$,~\eqref{mainEq33} is non-convex.

Adding the constraint $\norm{x}=1$ in~\eqref{mainEq33} yields the constrained $\ell_p$ regression in~\eqref{mainEq3}. Only recently, a pair of breakthrough results were suggested for solving~\eqref{mainEq3} if $p\neq 2$. \cite{park2017general} suggested a solution to the constraint $\ell_p$ regression in~\eqref{mainEq3} for the case $p=1$. They suggest to convert the constraint $\norm{x}=1$ into two inequality constraints $\norm{x} \leq 1$ and $-\norm{x} \leq 1$. The other result~\cite{jubran2018minimizing} suggested a provable constant approximation for the constrained $\ell_p$ regression problem, in time $O(n\log n)$ for every constant $p\geq1$. However, the result holds only for $d=2$, and the case $d>2$ as in our
 paper was left as an open problem. In fact, our main algorithm solves the problem recursively where in the base case $d=2$ we use the result from~\cite{jubran2018minimizing}.

To our knowledge, no existing provable approximation algorithms are known for handling outliers as in~\eqref{mainEq5}, unknown matching as in~\eqref{mainEq4}, or for~\eqref{mainEq3} for the case $p\in (0,1)$ and $d\geq 3$.

\paragraph{Coreset }for $\ell_p$ regression in this paper is a small weighted subset of the input that approximates $\norm{Ax-b}$ for every $x\in\REAL^d$, up to a multiplicative factor of $1\pm \eps$. Solving the constrained $\ell_p$ regression on such coreset would thus yield an approximation solution to the original (large) data. Such coresets of size independent of $n$ were suggested in~\cite{dasgupta2009sampling}. In Theorem~\ref{theorem:coreset} we obtain a little smaller coreset by combining~\cite{dasgupta2009sampling} and the framework from~\cite{FL11,braverman2016new}. Such coresets can also be maintained for streaming and distributed Big Data in time that is near-logarithmic in $n$ per point as explained e.g. in~\cite{feldman2011scalable,lucic2017training}. Applying our main result on this coreset, thus implies its streaming and distributed versions.
We note that this scenario is rare and opposite to the common case: in most coreset related papers, a solution for the problem that takes polynomial time $n^{O(1)}$ exists and the challenge is to reduces its running time to linear in $n$, by applying it on a coreset of size that is independent, or at least sub-linear in $n$. In our case, the coreset exists but not a polynomial time algorithm to apply on the coreset.

\section{Paper Overview}
The rest of the paper is organized as follows. We state our main contributions in Section~\ref{sec:contrib} and some preliminaries and notations in Section~\ref{sec:prelim}. Section~\ref{sec:lp_regression} suggests an approximation algorithm for solving~\eqref{mainEq3}, Section~\ref{sec:lp_regressionGeneral} generalizes this solution of~\eqref{mainEq3} to a wider range of functions, Section~\ref{sec:lp_regressionNoMatching} handles the minimization problem in~\eqref{mainEq4}, Section~\ref{sec:coreset} introduces a coreset for the constrained $\ell_p$ regression, Section~\ref{sec:ER} presents our experimental results and Section~\ref{sec:conclude} concludes our work.

\section{Our Contribution} \label{sec:contrib}
Some of the proofs have been placed in the appendix to make the reading of the paper more clear.

\paragraph{Constrained $\ell_p$ regression. }We provide the first polynomial time algorithms that approximates, \emph{with provable guarantees}, the functions in~\eqref{mainEq3}, \eqref{mainEq5} and~\eqref{mainEq4} up to some constant factor that depends only on $d$ and some error parameter $\varepsilon \in (0,1)$. The factor of approximation is $(1+\varepsilon)4^{d-1},4^{d-1}$ and $4^{d-1}$, respectively for~\eqref{mainEq3}, \eqref{mainEq5} and~\eqref{mainEq4}. The running time is $O(n\log{n}), n^{O(d)}$ and $n^{O(d)}$, respectively for~\eqref{mainEq3}, \eqref{mainEq5} and~\eqref{mainEq4}; see Table~\ref{table:ourContrib}.

%
%

\paragraph{Coresets. }Our main algorithm takes time $n^{O(d)}$ and is easily generalized for many objective functions via Observation~\ref{obs:distToCost}, Theorem~\ref{costAxb} and Theorem~\ref{costAxb_noMatch}. It implies the results in the last three rows of Table~\ref{table:ourContrib}.
For the case that $p\geq 1$ and we wish to minimize $\norm{Ax-b}_p$ over every unit vector $x$, we can apply our algorithm on the coreset for $\ell_p$ regression as explained in the previous section. This reduces the running time to near-linear in $n$, and enables parallel computation over distribution and streaming data by applying it on the small coreset that is maintained on the main server. This explains the running time and approximation factor for the first three lines of Table~\ref{table:ourContrib}.

\paragraph{Our experimental results }show that the suggested algorithms perform better, in both accuracy and computation time, compared to the few state of the art methods that can handle this non-convex problem; See Section~\ref{sec:ER}.

Table~\ref{table:ourContrib} summarizes the main contributions of this paper.

\begin{table*}[h!]
\begin{adjustbox}{width=1\textwidth}
\footnotesize
\footnotesize\tabcolsep=0.11cm \begin{tabular}{|c|c|c|c|c|}
\hline
\makecell{Function Name} & \makecell{Objective\\Function} & \makecell{Computation\\Time} & \makecell{Approximation\\Factor} & \makecell{Related\\Theorem}\\
\hline
\makecell{Constrained $\ell_p$ regression} & \eqref{mainEq3} & $O(n\log{n})$ & $(1+\varepsilon)\cdot4^{d-1}$ & \ref{costAxb} and \ref{theorem:coreset}\\
\hline
\makecell{Constrained $\ell_p^z$ regression} & $\displaystyle\min_{x\in \sphere^{d-1}} \norm{Ax-b}_p^z$ & $O(n\log{n})$ & $(1+\varepsilon)\cdot4^{s\cdot (d-1)}$ & \ref{costAxb} and \ref{theorem:coreset}\\
\hline
\makecell{Constrained $\ell_p$ regression\\with M-estimators} & $\displaystyle\min_{x\in\sphere^{d-1}}\sum_{i=1}^n \min\br{|a_i^Tx-b_i|,T}$ & $O(n\log{n})$ & $(1+\varepsilon)\cdot4^{d-1}$ & \ref{costAxb} and \ref{theorem:coreset}\\
\hline
\makecell{Constrained $\ell_p$\\regression with outliers} & \eqref{mainEq5} & $n^{O(d)}$ & $4^{d-1}$ & \ref{costAxb}\\
\hline
\makecell{Constrained $\ell_p$ regression\\with unknown matching} & \eqref{mainEq4} & $n^{O(d)}$ & $4^{d-1}$ $\star$ & \ref{costAxb_noMatch}\\
\hline
\end{tabular}
\end{adjustbox}
\caption{\textbf{Main results of this paper.} Let $d,p,z\in (0,\infty)$ be constants. We assume that $A = (a_1 \mid \cdots \mid a_n)^T \in \REAL^{n\times d}$ is a non-zero matrix of $n \geq d-1 \geq 1$ rows, and that $b = (b_1,\cdots,b_n)^T \in \REAL^n$. Let $T>0$, $\varepsilon \in (0,1)$ and $s = z$ if $z>1$ and $s=1$ otherwise.
The approximation factor is relative to the minimal value of the objective function over every unit vector $x\in\sphere^{d-1}$.
Rows marked with a $\star$ have that the minimum of the objective function is computed both over every $x\in\sphere^{d-1}$ and matching function $\M\in \perms(n)$.}
\label{table:ourContrib}
\end{table*}

\section{Preliminaries} \label{sec:prelim}
In this section we first give notation and main definitions that are required for the rest of the paper.

\paragraph{Notation.} \label{ch:notation}
Let $\REAL^{n\times d}$ be the set of $n\times d$ real matrices. In this paper, every vector is a column vector, unless stated otherwise, that is $\REAL^{d} = \REAL^{d\times 1}$. We denote by $\norm{p}=\norm{p}_2=\sqrt{p_1^2+\ldots+p_d^2}$ the length of a point $p=(p_1,\cdots,p_d) \in \REAL^d$, by $\dist(p,\pi) = \min_{x \in \pi} \norm{p-x}_2$ the Euclidean distance from $p$ to a subspace $\pi$ of $\REAL^d$.
We denote $[n]=\br{1,\cdots,n}$ for every integer $n \geq 1$.
For a bijection function $\M:[n]\to [n]$ (matching function) and a set $Y = \br{(a_1,b_1),\cdots,(a_n,b_n)} \subseteq \REAL^{d}\times\REAL$ of $n$ pairs, we define $Y_{\M} = \br{(a_1,b_{\M(1)}),\cdots,(a_n,b_{\M(n)})}$.
For every $i\in [d]$, we denote by $e_i \in \REAL^d$ the $i$th standard vector, i.e., the $i$th column of the $d\times d$  identity matrix. For $d\geq 2$, $\sphere^{d-1} = \br{x \in \REAL^d \mid \norm{x}=1}$ is the set of all unit vectors in $\REAL^d$. We denote $\sphere = \sphere^1$ for simplicity.
For every $x\in\REAL$, we define $\sign(x) = \begin{cases} 1 & \text{if } x\geq 0 \\ -1 & \text{otherwise} \end{cases}$.
The set $\perms(n)$ denotes the union over every matching function $\M:[n]\to[n]$.
We remind the reader that $\argmin_{x\in X} f(x)$ is the set (and not a scalar) that contains all the values of $x$ that minimize $f(x)$ over some set $X$.
We denote by $\vec{0}(d) = (0,\cdots,o)^T \in \REAL^d$.

\paragraph{Definitions. }We first give a brief geometric illustration of the later definitions.
A hyperplane $h$ in $\REAL^d$ that has distance $\hat{b}>0$ from the origin can be defined by its orthogonal (normal) unit vector $\hat{a}$ so that $h=\br{x\in\REAL^d\mid \hat{a}}^Tx-\hat{b}=0|$. More generally, the distance from $x\in\REAL^d$ to $h$ is $|\hat{a}^Tx-\hat{b}|$. In Definition~\ref{defOptAxb} below the input is a set $h_1,\cdots,h_m$ of $m$ such hyperplanes that are defined by the matrix $A = (a_1 \mid \cdot \mid a_m)^T \in\REAL^{m\times d}$ and vector $b = (b_1,\cdots,b_m)^T\in\REAL^m$, such that $\frac{a_i}{\norm{a_i}}$is the unit normal to the $i$th hyperplane, and $\frac{b_i}{\norm{a_i}}$ is its distance from the origin.

In what follows we define a (possibly infinite) set of \emph{unit vectors} $\opt(A,b)$, which are the unit vectors that are as close as possible to the hyperplane $h_m$, among all unit vectors that lie on the intersection of $h_1,\cdots,h_{m-1}$. If $m=1$, we define this set to be the set of closest vectors on the unit sphere to the hyperplane $h_1$.
Observe that for every point $x\in \REAL^d$, $x \in h_i \iff a_i^Tx=b_i$; See Figure~\ref{Alg2DCase} and~\ref{AlgDCase} for a geometric illustration of the following definition.
\begin{definition} \label{defOptAxb}
Let $m,d\geq 1$ be a pair of integers, $A = (a_1 \mid \cdots\mid a_m)^T \in \REAL^{m\times d}$ and $b = (b_1,\cdots,b_m) \in [0,\infty)^m$. We define the set
\[
\begin{split}
&\opt(A,b):=\\
&\begin{cases}
  \displaystyle \argmin_{x\in\sphere^{d-1}} |a_m^Tx-b_m|, & \mbox{if } m=1 \\
  \displaystyle \argmin_{\substack{x\in\sphere^{d-1}: \\(a_1 \mid \cdots\mid a_{m-1})^Tx=(b_1,\cdots,b_{m-1})^T}} |a_m^Tx-b_m|, & \mbox{otherwise}.
\end{cases}
\end{split}
\]
\end{definition}

\begin{figure}
\begin{subfigure}[h]{0.49\textwidth}
		\centering
		\includegraphics[width=0.9\textwidth]{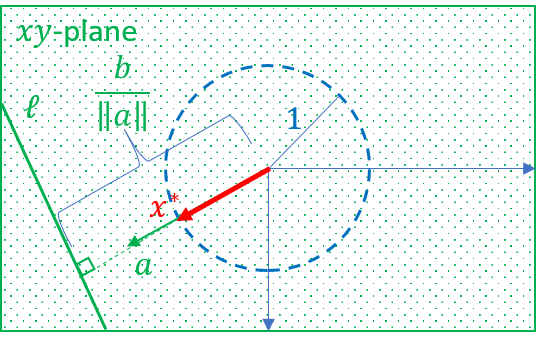} \label{Alg2DCase1}
    \caption{\it $b \geq \norm{a}$. Hence, the unit circle has only $1$ point $x^*$ (red) with minimal distance to $\ell$.}
	\end{subfigure}
	\begin{subfigure}[h]{0.49\textwidth}
		\centering\centering		
		\includegraphics[width=0.9\textwidth]{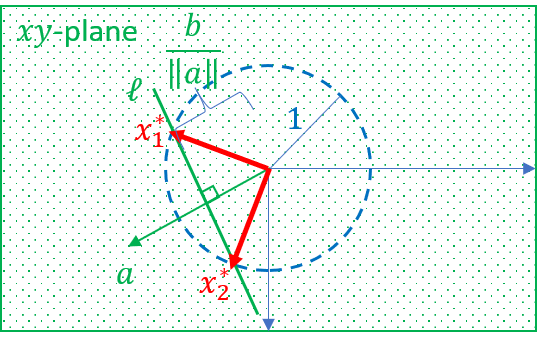} \label{Alg2DCase2}
		\caption{\it $b < \norm{a}$. Hence, the unit circle has exactly $2$ intersection points $x_1^*$ and $x_2$ (red) with $\ell$.}
	\end{subfigure}
	\caption{\it For $d=2$, $a\in \REAL^2$ and $b\in [0,\infty)$, the set $\ell = \br{p\in \REAL^2 \mid \frac{a^T}{\norm{a}}p=\frac{b}{\norm{a}}}$ is a line whose normal is $\frac{a}{\norm{a}}$ and is of distance $\frac{b}{\norm{a}}$ from the origin (Green line). The set $\opt(a,b)$, which contains unit vectors of minimal distance to $\ell$ are shown in red.}
	\label{Alg2DCase}
\end{figure}
\begin{figure}[t]
    \begin{subfigure}[h]{0.23\textwidth}
		\centering
		\includegraphics[width=0.8\textwidth]{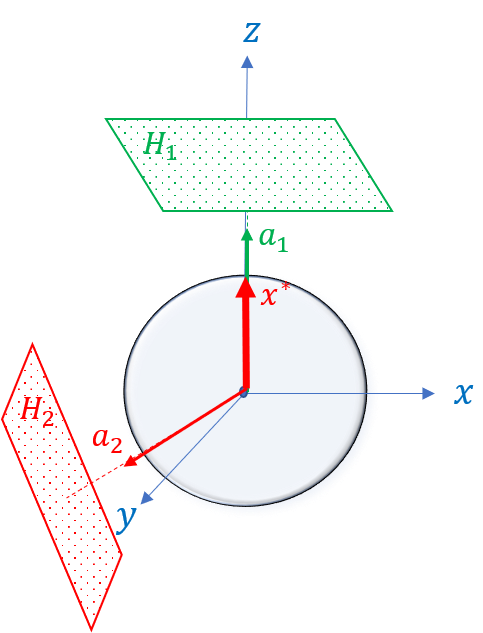} \label{AlgDCase1}
        \caption{\scriptsize $b_1 \geq \norm{a_1}$. Hence, $\sphere^{d-1}$ has $1$ point $x^*$ (red) of minimal distance to $H_1$. Thus, $\opt(a_1,b_1) = \br{x^*}$.}
	\end{subfigure}
	\begin{subfigure}[h]{0.23\textwidth}
		\centering\centering		
		\includegraphics[width=0.8\textwidth]{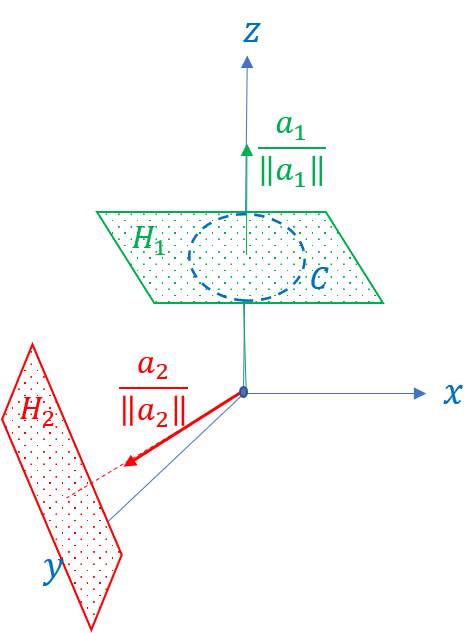} \label{AlgDCase2_3}
		\caption{\scriptsize $b_1 < \norm{a_1}$. Hence, $\opt(a_1,b_1) = C = H_1\cap S^{d-1}$, $|\opt(a_1,b_1)| = \infty$, and $\opt((a_1\mid a_2),(b_1,b_2)) = \argmin_{x\in \opt(a_1,b_1)} \dist(x,H_2) = \argmin_{x\in C} \dist(x,H_2)$.}
	\end{subfigure}
    \hfill
    \begin{subfigure}[h]{0.23\textwidth}
		\centering
		\includegraphics[width=0.9\textwidth]{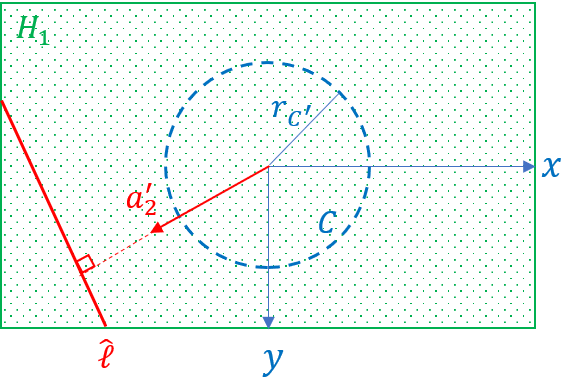} \label{AlgDCase2_4}
    \caption{\scriptsize A top view of $H_1$. $\hat{\ell}$ (red) is the intersection line between $H_1$ and $H_2$. $r_{C} \leq 1$ denotes the radius of the circle $C$. The normal $a_2'$ of $\hat{\ell}$ in the plane $H_1$ is simply the projection of $a_2$ onto the $xy$-plane.}
	\end{subfigure}
	\begin{subfigure}[h]{0.23\textwidth}
		\centering\centering		
		\includegraphics[width=0.9\textwidth]{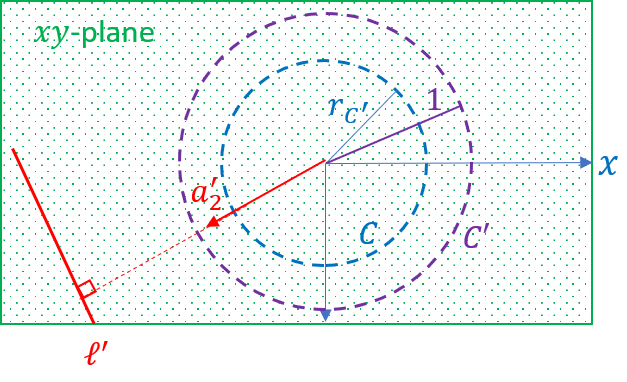} \label{AlgDCase2_5}
		\caption{\scriptsize Projecting $C'$ and $\hat{\ell}$ onto the $xy$-plane, and scaling by $1/r_{C'}$ to obtain a unit circle $C'$ (purple). $\ell'$ is the projected and scaled $\hat{\ell}$. $\opt((a_1\mid a_2),(b_1,b_2)) = \argmin_{x\in C'} \dist(x,\ell') = \br{a_2'}$, where the last equality is similar to the case $d=2$; See Figure~\ref{Alg2DCase}.}
	\end{subfigure}
	\caption{\scriptsize Let $d=3$, $a_1,a_2\in \REAL^d$ and $b_1,b_2\in [0,\infty)$ such that $b_2 \geq \norm{a_2}$. Assume that $a_1$ spans the $z$-axis. Let $H_i = \br{p\in \REAL^d \mid \frac{a_i^T}{\norm{a_i}}p=\frac{b_i}{\norm{a_i}}}$ for $i \in \br{1,2}$ be two hyperplane (Green and red planes).}
	\label{AlgDCase}
\end{figure}

In what follows, for every pair of vectors $v=(v_1,\cdots,v_n)$ and $u=(u_1,\cdots,u_n)$ in $\REAL^n$ we denote $v\leq u$ if $v_i \leq u_i$ for every $i\in [n]$. The function $f:\REAL^n \to [0,\infty)$ is \emph{non-decreasing} if $f(v)\leq f(u)$ for every $v\leq u$. For a set $I$ in $\REAL^d$, and a scalar $c\in\REAL$ we denote $I/c=\br{x/c\mid x\in I}$.

The following definition is a generalization of Definition 2.1 in~\cite{feldman2012data} from $n=1$ to $n>1$ dimensions, and from $\REAL$ to $I\subseteq\REAL^n$. Intuitively, an $r$-log-Lipschitz function is a function whose derivative may be large but cannot increase too rapidly (in a rate that depends on r).
\begin{definition}[Log-Lipschitz function] \label{def:lip}
Let $r>0$ and let $n\geq 1$ be an integer. Let $I$ be a subset of $\REAL^n$, and $h:I\to[0,\infty)$ be a non-decreasing function.
Then $h(x)$ is \emph{$r$-log-Lipschitz} over $x\in I$, if for every $c \geq 1$ and $x\in \displaystyle I \cap \frac{I}{c}$, we have $h(c x)\leq c^r h(x).$ The parameter $r$ is called the \emph{log-Lipschitz constant}.
\end{definition}

The following definition implies that we can partition a function $g$ which is not a log-Lipschitz function into a constant number of log-Lipschitz functions; see Figure~\ref{figLip} for an illustrative example.
\begin{definition}[Piecewise log-Lipschitz~\cite{jubran2018minimizing}\label{def:PieceLip}]
Let $g:X\to[0,\infty)$ be a continuous function over a set $X$, and
let $(X,dist)$ be a metric space, i.e. $\dist:X^2\to[0,\infty)$ is a distance function. Let $r> 0$. The function $g$ is \emph{piecewise $r$-log-Lipschitz} if there is a partition of $X$ into $m$ disjoint subsets $X_1,\cdots, X_m$
such that for every $i\in[m]$:
\vspace{-0.5cm}
\begin{enumerate}[noitemsep]
\renewcommand{\labelenumi}{(\roman{enumi})}
\item $g$ has a unique infimum $x_i$ at $X_i$, i.e., $\br{x_i}= \argmin_{x\in X_i}g(x)$.
\item $h_i:[0,\max_{x\in X_i} \dist(x,x_i)]\to[0,\infty)$ is an $r$-log-Lipschitz function; see Definition~\ref{def:lip}.
\item $g(x)=h_i(\dist(x_i,x))$ for every $x\in X_i$.
\end{enumerate}
\vspace{-0.4cm}
The set of minima is denoted by $M(g)=\br{x_1,\cdots,x_m}$.
\end{definition}

\begin{figure} [ht]
\centering
\includegraphics[width=0.25\textwidth]{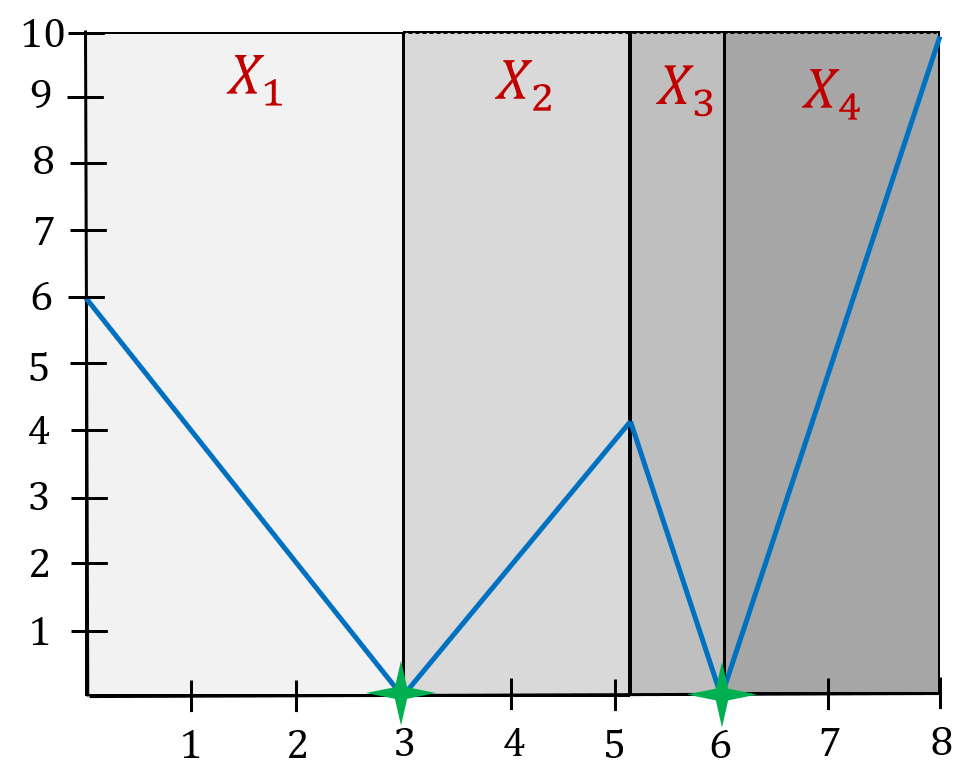}
\caption{\textbf{Piecewise log-Lipschitz function example. }\\
A function $g(x) = \min \br{2\cdot |x-3|, 5\cdot |x-6|}$ (blue graph) over the set $X = \REAL$. $X$ can be partitioned into $4$ subsets $X_1,\cdots,X_4$, where each subset has a unique infimum $x_1=3,x_2=3,x_3=6$ and $x_4=6$ respectively (green stars). There exist $4$ $1$-log-Lipschitz functions $h_1(x)=h_2(x)=2x$ and $h_3(x)=h_4(x)=5x$, such that $g(x) = h_i(|x_i-x|)$ for every $x\in X_i$. The figure is taken from~\cite{jubran2018minimizing}.}
\label{figLip}
\end{figure}

\section{Regression with a Given Matching} \label{sec:lp_regression}
In this section we suggest our main approximation algorithm for solving~\eqref{mainEq3}, i.e., when the matching between the rows of $A$ and the entries $b = (b_1,\cdots,b_n)^T$ are given.

The following two corollaries lies in the heart of our main result.
\begin{corollary} [Claims 19.1 and 19.2 in~\cite{jubran2018minimizing}] \label{corCosLip}
Let $b \geq 0$ and let $g:[0,2\pi) \to [0,\infty)$ such that $g(\alpha) = |\sin(\alpha) - b|$. Then $g$ is a piecewise $2$ log-Lipschitz function; See Definition~\ref{def:PieceLip}.
\end{corollary}

\begin{corollary} \label{cor4ApproxVectors}
Let $a \in \REAL^2\setminus \br{0}$ and $b \geq 0$. Let $y \in \argmin_{x\in \sphere^1} |a^T x - b|$. Then for every $u^*,u' \in \sphere$ such that $\norm{u'-y} \leq 2\cdot \norm{u^*-y}$ we have
\[
|a^Tu'-b| \leq 4\cdot |a^Tu^*-b|.
\]
\end{corollary}
\begin{proof}
See proof of Corollary~\ref{cor4ApproxVectors_proof} in the appendix.
\end{proof}

In what follows, a non-zero matrix is a matrix of rank at least one (i.e., not all its entries are zero).
Let $A = (a_1 \mid \cdots \mid a_n)^T \in \REAL^{n\times d}$ and $b = (b_1,\cdots,b_n) \in [0,\infty)^n$, where we assume that $A$ is a non-zero matrix.
For every $i\in [n]$, let $h_i$ be a hyperplane whose normal is $\br{\frac{a_i}{\norm{a_i}}}_{i=1}^n$ and its distances from the origin is $\br{\frac{b_i}{\norm{a_i}}}_{i=1}^n$, and let $H = \br{h_i}_{i=1}^n$ denote their union.
Let $x^*\in\REAL^d$ be a unit vector. The following lemma generalizes Corollary~\ref{cor4ApproxVectors} to higher dimensions. It states that there is a unit vector $x'$ of minimal distance $\dist(h_k,x')$ to one of the $n$ hyperplanes $h_k$, that approximates $\norm{a_i}\cdot \dist(h_i,x^*) = |a_i^Tx^*-b_i|$ for every $i\in [n]$ up to a multiplicative factor of $4$.
\begin{lemma} \label{corAxb4Approx}
Let $A = (a_1 \mid \cdots \mid a_n)^T \in \REAL^{n\times d} \setminus \br{0}^{d\times d}$ be a non-zero matrix such that $n \geq d-1 \geq 1$ points, and let $b = (b_1,\cdots,b_n)^T \in [0,\infty)^n$ and $x^* \in \sphere^{d-1}$. Then there exists $j\in [n]$ where $\norm{a_j} \neq 0$ and $x_j \in \argmin_{x\in \sphere^{d-1}} |a_j^Tx-b_j|$ such that for every $i\in [n]$
\[
|a_i^Tx_j-b_i| \leq 4\cdot |a_i^Tx^*-b_i|.
\]
\end{lemma}
\begin{proof}
See proof of Lemma~\ref{corAxb4Approx_proof} in the appendix.
\end{proof}

The following lemma states that there is a set $\br{h_1,\cdots,h_m} \subseteq H$ of $m\leq d-1$ hyperplanes from $H$, and a unit vector $x'$ in the intersection of the first $m-1$ hyperplanes that is closest to the last hyperplane, $x'\in\argmin_{x\in \bigcap_{i\in[m-1]}h_i} |a_m^Tx-b_m|$, that is closer to every one of the $n$ hyperplane in $H$, up to a factor of $4^{d-1}$ than $x^*$, i.e., $|a_i^Tx'-b_i| \leq 4^{d-1}\cdot|a_i^Tx^*-b_i|$ for every $i\in [n]$.
The proof is based on applying Lemma~\ref{corAxb4Approx} recursively $m$ times. In the following lemma we use $opt$ from Definition~\ref{defOptAxb}.
\begin{lemma} \label{lemApproxD}
Let $A = (a_1 \mid \cdots \mid a_n)^T \in \REAL^{n\times d}$ be a non-zero matrix of $n \geq d-1 \geq 1$ rows, let $b = (b_1,\cdots,b_n)^T \in [0,\infty)^n$, and let $x^* \in \sphere^{d-1}$.
Then there is a set $\br{{i_1},\cdots,{i_{r}}} \subseteq [n]$ of $r \in [d-1]$ indices such that for $X = \opt((a_{i_1}\mid\cdots\mid a_{i_{r}}), (b_{i_1},\cdots,b_{i_{r}}))$ and every $i\in [n]$, there is $x'\in X$ that satisfies
\begin{equation} \label{mainToProve_proof}
|a_i^Tx'-b_i| \leq 4^{d-1} \cdot |a_i^Tx^*-b_i|.
\end{equation}
Moreover, \eqref{mainToProve_proof} holds for every $x'\in X$ if $|X| = \infty$.
\end{lemma}
\begin{proof}
See proof of Lemma~\ref{lemApproxD_proof} in the appendix.
\end{proof}

\subsection{Computing the set $\opt$} \label{sec:compOpt}
In this section we give a suggested simple implementation for computing the set $\opt(A,b)$ given a matrix $A\in\REAL^{m\times d}$ and $b\in [0,\infty)^m$; see Algorithm~\ref{calcOpt}.
A call to $\algCalcOpt(A,b)$ returns $\opt(A,b)$ if $|\opt(A,b)| \in O(1)$. Otherwise, it returns some $x\in\opt(A,b)$. In other words, the algorithm always returns a set of finite size.

\setcounter{AlgoLine}{0}
\begin{algorithm}
	\caption{\textsc{$\algCalcOpt(A,b)$}}
	 \label{calcOpt}
	\SetKwInOut{Input}{Input}
	\SetKwInOut{Output}{Output}
	\Input{$A=(a_1\mid\cdots\mid a_m)^T \in \REAL^{m\times d}$ where $m \in [d-1]$ and $\norm{a_i} \neq 0$ for every $i\in [m]$, and a vector $b=(b_1,\cdots,b_m) \in [0,\infty)^m$.}
	\Output{The set $\opt(A,b)$ if its size $|\opt(A,b)|$ is finite, and arbitrary $x\in \opt(A,b)$ otherwise.}

    \uIf {$b_1 \geq \norm{a_1}$ \label{checkCase1}}{
        \uIf {$m = 1$}{
            \tcc{$\displaystyle\argmin_{x\in \sphere^{d-1}}|a_1^Tx-b_1| = \br{\frac{a_1}{\norm{a_1}}}$.}
            Set $X \gets \br{\frac{a_1}{\norm{a_1}}}$. \label{linem1_1}
        } \Else {
            Set $X \gets \emptyset$. \tcp{The minimizer $x$ of $|a_1^Tx-b_1|$ does not satisfy that $(a_2\mid\cdots\mid a_{m-1})^Tx=(b_2,\cdots,b_{m-1})^T$. Hence, $\opt(A,b) = \emptyset$.}
        }
    } \Else {
        \uIf{$m = 1$}{
            \uIf{$d=2$} {
                Set $X \gets \br{x \in \sphere \mid a_1^Tx = b_1}$. \label{linem1_2}  \tcp{There are $2$ unit vectors that satisfy $a_1^Tx=b_1$ in 2D.}
            } \Else {
                Set $X \gets \emptyset$. \label{linem1_3} \tcp{There are infinite minimizers, and no more constraints.}
            }
        } \Else {
            Set $I \gets$ the $d$-dimensional identity matrix.\\
            Set $I' \gets$ the first $(d-1)$ rows of $I$.\\
            Set $R$ to be a rotation matrix such that $\frac{Ra_1}{\norm{a_1}} = e_d$. \label{lineRota1} \tcp{The $d$-dimensional rotation matrix that rotates $\frac{a_1}{\norm{a_1}}$ to the $d$th standard vector.}
            \For {every $i\in [m]\setminus \br{1}$} {
                Set $a_i' \gets I'Ra_i$. \tcp{Project $Ra_i$ onto the hyperplane orthogonal to $e_d$.}
                Set $b_i' \gets \frac{-b_1\cdot\frac{(Ra_i)^Te_d}{\norm{a_1}}+b_i}{\sqrt{1-\frac{b_1^2}{\norm{a_1}^2}}}$.\\
            }
            Set $A' \gets (\sign(b_2')\cdot a_2' \mid \cdots\mid \sign(b_m')\cdot a_m')^T$.

            Set $b' \gets (\sign(b_2')\cdot b_2' \mid \cdots\mid \sign(b_m')\cdot b_m')$.

            \uIf{$A'$ contains only zero entries}{
                Set $X \gets \br{R^T\left(x'^T \mid \frac{b_1}{\norm{a_1}}\right)^T}$ for arbitrary $x'\in\REAL^{d-1}$ such that $\norm{x'}=\sqrt{1-\frac{b_1^2}{\norm{a_1}^2}}$.
            } \Else{
                Set $X' \gets \algCalcOpt(A',b')$.\\
                \tcp{Recursive call, where $A' \in \REAL^{(m-1)\times (d-1)}$ and $b' \in \REAL^{m-1}$.}
                Set $X \gets$\\
                $\br{R^T\left(\sqrt{1-\frac{b_1^2}{\norm{a_1}^2}}\cdot x'^T \mid \frac{b_1}{\norm{a_1}}\right)^T \mid x' \in X'}$. \label{lineCompRecursiveX}   \tcp{$R^T$ is the inverse of $R$ from Line~\ref{lineRota1}.}
            }
        }
    }
  	\Return $X$.
\end{algorithm}

\subsection{Geometric interpretation and intuition behind Algorithm~\ref{calcOpt}. }Algorithm~\ref{calcOpt} takes a set $h_1,\cdots,h_m$ of $m$ hyperplanes in $\REAL^d$ as input, each represented by its normal and
its distance from the origin, and computes a point $x\in \sphere^{d-1} \cap \bigcap_{i\in[m-1]} h_i$ that minimizes its distance to $h_m$, i.e., $\dist(x,h_m)$. In other words, among all vectors
that lie simultaneously on $h_1,\cdots,h_{m-1}$, we intend to find the unit vector $x$ which minimizes its distance to $h_m$. There can either be $0$,$1$, $2$ or infinite such points.
In an informal high-level overview, the algorithm basically starts from some unit vector $x\in\REAL^d$, and rotates it until it either intersects $h_1$, or minimizes its distance to $h_1$ without intersecting it.
If they intersect, then we rotate $x$ while maintaining that $x\in h_1$, until $x$ either intersects $h_2$ or minimizes its distance to $h_2$. If $x$ intersects $h_2$, we rotate $x$ in the intersection $h_1\cap h_2$ until it intersects $h_3$ and so on.
If at some iteration we observe that all the remaining subspaces are parallel, then the set $\opt(A,b)$ is of infinite size, so we return one element from it.
We stop this process when at some step $k$, we can not intersect $h_k$ under the constraint that $x$ is a unit vector and $x\in \bigcap_{i\in[k-1]}h_i$. If $k<m$, then $\opt(A,b)$ is empty. If $k=m$ we return the vector $x$ that is as close as possible to $h_m$.

More formally, at each step of the algorithm we do the following. In Line~\ref{checkCase1} we check weather the intersection $h_1\cap\sphere^{d-1}$ contains at most $1$ point. This happens when the distance $b_1/\norm{a_1}$ of the
hyperplane $h_1$ to the origin is bigger than or equal to $1$. This is the simple case in which we terminate. The interesting case is when the distance of the hyperplane to the origin is less than $1$.
In this case, if $d=2$, we compute and return the two possible intersection points in $h_1\cap\sphere^{d-1}$. If $d\geq 3$, then $|h_1\cap\sphere^{d-1}|=\infty$. Observe that $h_1\cap\sphere^{d-1}$ is simply
a sphere $\sphere' \subseteq h_1$ of dimension $d-2$, but is not a unit sphere. We rotate the coordinates system such that the hyperplane containing $\sphere'$ is orthogonal to the $d$th standard vector $e_d$. We observe that
in order to minimize the distance from a point in $\sphere'$ to $h_2$, we need to minimize its distance to $h_1\cap h_2$ which is simply a $(d-2)$-subspace contained in $h_1$.
We thus project $\sphere'$ and every $(d-2)$-subspace in $\br{h_1\cap h_i \mid i\in [m] \setminus \br{1}}$ onto the hyperplane $H$ orthogonal to $e_d$ and passes through the origin, obtaining a sphere $\sphere''$ of dimension $d-2$, and $m-1$
$(d-2)$-subspaces, all contained in $H$. We scale the system such that $\sphere''$ becomes a unit sphere. We then continue recursively.

\paragraph{Overview of Algorithm~\ref{calcX}.} The input is a matrix $A$ and a vector $b$. The output is a set of unit vectors that satisfies Theorem~\ref{Axb}. We assume without loss of generality that the entries of $b$ are non-negative, otherwise we change the corresponding signs of rows in $A$ in Lines~\ref{line:SignFor}-\ref{line:endSignFor}.
In Lines~\ref{line:ES}-\ref{line:ESIndices} we run exhaustive search over every possible set of at most $d-1$ indices, which corresponds to a set $S$ of unit vectors, to find the set $X$ from Lemma~\ref{lemApproxD}.
The algorithm then returns the union of these sets in Line~\ref{line:addToX}.
See Algorithm~\ref{calcOpt} for a possible implementation for Line~\ref{line:compOpt}. Notice that Algorithm~\ref{calcOpt} does not return a set of infinite size. In such a case, it returns one element from the infinite set.

\setcounter{AlgoLine}{0}
\begin{algorithm}[tb]
    \caption{\textsc{$\algnameXCandidates(A,b)$}}
    \label{calcX}
    \SetKwInOut{Input}{Input}
	\SetKwInOut{Output}{Output}
    \Input{A matrix: $A=(a_1\mid\cdots\mid a_n)^T \in \REAL^{n\times d}$ and a vector $b=(b_1,\cdots,b_n) \in \REAL^n$ where $n \geq d-1 \geq 1$.}
    \Output{A set $X \subseteq \sphere^{d-1}$; See Theorem~\ref{Axb}.}

    Set $X \gets \emptyset$

    \For{every $i \in [n]$ \label{line:SignFor}}{
        Set $b_{i}' \gets |b_{i}|$

        Set $a_{i}' \gets \sign(b_{i})\cdot a_{i}$ \label{line:endSignFor}
    }
    \For {every $r \in [d-1]$ \label{line:ES}}{
        \For {every distinct set $\br{i_1,\cdots,i_r} \subseteq [n]$ where $\norm{a_{i_k}}\neq 0$ for every $k\in [r]$ \label{line:ESIndices}}{
        	Set $S \gets \algCalcOpt((a_{i_1}'\mid\cdots\mid a_{i_r}'),(b_{i_1}',\cdots,b_{i_r}'))$. \label{line:compOpt} \tcp{See Definition~\ref{defOptAxb} and Algorithm~\ref{calcOpt}.}
        	Set $X \gets X \cup S$ \label{line:addToX}\\ 	
      	}
    }
    \Return $X$
\end{algorithm}

What follows is the main theorem of Algorithm~\ref{calcX}.
The following theorem states that the output $X$ of Algorithm~\ref{calcX} contains the desired solution.
This is by Lemma~\ref{lemApproxD} that ensures that one of the sets $S$ contains the desired solution. 
\begin{theorem} \label{Axb}
Let $A = (a_1 \mid \cdots \mid a_n)^T \in \REAL^{n\times d}$ be a matrix of $n \geq d-1 \geq 1$ rows, and let $b = (b_1,\cdots,b_n)^T \in \REAL^n$.
Let $X \subseteq \sphere^{d-1}$ be an output of a call to \algnameXCandidates$(A,b)$; see Algorithm~\ref{calcX}. Then for every $x^* \in \sphere^{d-1}$ there exists a unit vector $x' \in X$ such that for every $i\in[n]$,
\[
|a_i^Tx'-b_i| \leq 4^{d-1}\cdot |a_i^Tx^*-b_i|.
\]
Moreover, the set $X$ can be computed in $n^{O(d)}$ time and its size is $|X| \in n^{O(d)}$.
\end{theorem}
\begin{proof}
See proof of Theorem~\ref{Axb_proof} in the appendix.
\end{proof}

\subsection{Generalization} \label{sec:lp_regressionGeneral}
We now prove that the output of Algorithm~\ref{calcX} contains approximations for a large family of optimization functions. Note that each function may be optimized by a different candidate in $X$.
This family of functions includes squared distances, $M$-estimators and handling $k\geq 1$ outliers. It is defined precisely via the following cost function, where the approximation error depends on the parameters $r$ and $s$.
\begin{definition} [Definition 4 in~\cite{jubran2018minimizing}]\label{def:cost}
Let $Y = \br{y_1,\cdots,y_n}$ be a finite input set of elements and let $Q$ be a set of queries.
Let $D:X\times Q\to [0,\infty)$ be a function.
Let $\lip:[0,\infty)\to [0,\infty)$ be an $r$-log-Lipschitz function and $f:[0,\infty)^n\to [0,\infty)$ be an $s$-log-Lipschitz function.
Let $q\in Q$. We define
\[
\cost(Y,q)=f\left( \lip \left(D\left(y_1,q\right)\right),\cdots, \lip\left(D\left(y_n,q\right)\right)\right).
\]
\end{definition}
In what follows, $Y = \br{(a_1,b_1),\cdots,(a_n,b_n)} \subseteq \REAL^d \times \REAL$, $Q = \sphere^{d-1}$ will be the set of all unit vectors in $\REAL^d$, and $D((a,b),x) = |a^Tx-b|$.
Table~\ref{table:examples} gives some examples of different cost functions that satisfy the requirements of Definition~\ref{def:cost}.

\begin{table}[t]
\begin{adjustbox}{width=0.49\textwidth}
\small
\begin{tabular}{ | c | c | c | c |}
\hline
Use case & \makecell{Optimization\\Problem} & $f(v)$ & $\lip(x)$\\
\hline
\makecell{Constrained\\$\ell_2$ regression} & $\displaystyle\min_{x\in\sphere^{d-1}}\norm{Ax-b}_2$ & $\norm{v}_2$ & $x$\\
\hline
\makecell{Constrained\\$\ell_p$ regression} & $\displaystyle\min_{x\in\sphere^{d-1}}\norm{Ax-b}_p$ & $\norm{v}_p$ & $x$\\
\hline
\makecell{Constrained\\$\ell_p$ regression\\with noisy data} & $\displaystyle\min_{x\in\sphere^{d-1}}\left(\sum_{i=1}^n \min\br{|a_i^Tx-b_i|^p,T^p}\right)^{1/p}$ & $\norm{v}_p$ & \makecell{$\min\br{x,T}$,\\ $T>0$}\\
\hline
\makecell{Constrained\\$\ell_p$ regression\\with outliers} & $\displaystyle\min_{x\in\sphere^{d-1}}\norm{\smallest(Ax-b)}_p$ & \makecell{$\norm{\smallest(v,n-k)}_p$,\\ $k<n$} & $x$\\
\hline
\end{tabular}
\end{adjustbox}
\caption{\textbf{Examples for Definition~\ref{def:cost}. } Let $Y = \br{(a_1,b_1),\cdots,(a_n,b_n)} \subseteq \REAL^d \times \REAL$. Algorithm~\ref{calcX} approximates the minimum of $\cost(Y,x)=f\left( \lip \left(|a_1^Tx-b_1|\right),\cdots, \lip\left(|a_n^Tx-b_n|\right)\right)$. Possible optimization functions $f(v)$ and $lip(x)$ are suggested in this table. Let $\smallest(v,n-k) \in \REAL^k$ denote the $n-k$ smallest entries of a vector $v\in\REAL^n$}
\label{table:examples}
\end{table}

The following observation states that if we find a query $q\in Q$ that approximates the function $D$ for every input element, then it also approximates the function $\cost$ as defined in Definition~\ref{def:cost}.
\begin{observation} [Observation 5 in~\cite{jubran2018minimizing}] \label{obs:distToCost}
Let $\cost (Y,q)=f\left( \lip \left(D\left(y_1,q\right)\right),\cdots, \lip\left(D\left(y_n,q\right)\right)\right)$ be defined as in Definition~\ref{def:cost}.
Let $q^*,q' \in Q$ and let $c\geq 1$. If $D\left(y_i,q'\right)\leq c\cdot D\left(y_i,q^*\right)$ for every $i\in [n]$, then
\[
\cost\left(Y,q'\right)\leq c^{rs} \cdot \cost\left(Y,q^*\right).
\]
\end{observation}

The optimal solution in the following theorem can be computed by taking the optimal solution for the corresponding cost function at hand among the output set $X$ of candidates from Algorithm~\ref{calcX}.
\begin{theorem} \label{costAxb}
Let $A = (a_1 \mid \cdots \mid a_n)^T \in \REAL^{n\times d}$ be a matrix of $n \geq d-1 \geq 1$ rows, and let $b = (b_1,\cdots,b_n)^T \in \REAL^n$. Let $\cost, s, r$ be as defined in Definition~\ref{def:cost} for $Y = \br{(a_i,b_i)\mid i\in [n]}$ and $D((a,\hat{b}),x) = |a^Tx-\hat{b}|$ for every $a\in\REAL^d$, $\hat{b} \in \REAL$ and $x\in \REAL^d$. Then in $n^{O(d)}$ time we can compute a unit vector $x'\in\sphere^{d-1}$ such that
\[
\cost(Y,x') \leq 4^{(d-1)rs} \cdot \min_{x\in\sphere^{d-1}} \cost(Y,x).
\]
\end{theorem}
\begin{proof}
See proof of Theorem~\ref{costAxb_proof} in the appendix.
\end{proof}

\section{Handling Unknown Matching} \label{sec:lp_regressionNoMatching}
In this section, we tackle the constrained $\ell_p$ regression problem with unknown matching between the rows of $A$ and entries of $b$. That is, where the minimum of~\eqref{mainEq3} is not only over every nit vector $x$, but also over every permutation of the entries of $b$ (or rows in $A$). Thus, the problem now at hand is to solve~\eqref{mainEq4}.
The main result is derived from the following simple observation. In Lemma~\ref{lemApproxD} we proved there is a set of $r\leq d-1$ \textbf{matching} subsets $\hat{A}$ and $\hat{b}$ of $A$ and $b$ respectively, such that $\opt(\hat{A},\hat{b})$ contains the desired unit vector $x'$. This result still holds when the matching is unknown, though we dont know whose the subset $\hat{b}$ in $b$ that corresponds to $\hat{A}$. Hence, for every such $\hat{A}$, all we have to do is apply another exhaustive search over subset of $b$ in order to find the subset $\hat{b}$ that best matches $\hat{A}$.

Algorithm~\ref{Alg_NoMatching} uses the following definition of optimal matching between an input set of pairs $Y$, a query $q$, and a cost function.
\begin{definition} [\textbf{Optimal matching}] \label{def:optimalMatching}
Let $\perms(n)$ denote the union over every matching function (bijection function) $\M:[n]\to[n]$ and let $Q$ be a set. For every $i\in [n]$, let $y_i = (a_i,b_i)$ be a pair of elements, and let $Y = \br{y_1,\cdots,y_n}$ be their union.
Consider a function $\cost$ as defined in Definition~\ref{def:cost} for $f(v) = \norm{v}_1$ and let $q \in Q$.
We denote by $\hat{\M}(Y,q,\cost)$ the matching function that minimizes $\cost(Y_{\M},q)$ over every permutation $\M \in \perms(n)$. Formally,
\[
\hat{\M}(Y,q,\cost) \in \argmin_{\M \in \perms(n)} \cost(Y_{\M},q).
\]
\end{definition}

\vspace*{-0.6cm}
\paragraph{Overview of Algorithm~\ref{Alg_NoMatching}.} In Lines~\ref{line:ES2}-\ref{line:addToX2}, we iterate over every pair of sets $I=\br{i_1,\cdots,i_r}$ and $\ell=\br{\ell_1,\cdots,\ell_r}$ of $r$ indices, for $r\leq d-1$. We assume that $a_{i_k}$ matches $b_{\ell_k}$ for every $k\in [r]$. We change the signs of $a_{i_k}$ and $b_{\ell_k}$ as in Algorithm~\ref{calcX}. We make a call to Algorithm~\ref{calcOpt} to compute and store in $S$ the set $\opt(a_{i_1}'\mid\cdots\mid a_{i_r}'),(b_{\ell_1}',\cdots,b_{\ell_r}')$. Notice that we do not use infinite sets in Algorithm~\ref{calcOpt}. We then add $S$ to $X$;
In Lines~\ref{line:SetY}-\ref{line:compOptMatch} we compute for every $x\in X$ the matching $\hat{\M}(Y,x,\cost)$ that minimizes the given cost function $\cost$, and store the pair in $S$. In Line~\ref{line:compBestPair} we return the pair in $S$ whose cost is minimal.

\setcounter{AlgoLine}{0}
\begin{algorithm}[tb]
    \caption{\textsc{$\matchingalgname(A,b, \cost)$}}
    \label{Alg_NoMatching}
    \SetKwInOut{Input}{Input}
	\SetKwInOut{Output}{Output}
    \Input{$A = (a_1 \mid \cdots \mid a_n)^T \in \REAL^{n\times d}$, $b = (b_1,\cdots,b_n)^T \in \REAL^n$, and cost as in Theorem~\ref{costAxb_noMatch}.}
    \Output{A pair $(\tilde{x},\tilde{\M})$ of a unit vector and a matching funtion; See Theorem~\ref{costAxb_noMatch}.}

    Set $X \gets \emptyset$

    \For {every $r \in [d-1]$ \label{line:ES2}}{
        \For {every distinct sets $\br{i_1,\cdots,i_r},\br{\ell_1,\cdots,\ell_r} \subseteq [n]$ where $\norm{a_{i_k}}\neq 0$ for every $k\in [r]$}{
            \For {every $k \in [r]$ \label{line:SignFor2}}{
                Set $b_{\ell_k}' \gets |b_{\ell_k}|$

                Set $a_{i_k}' \gets \sign(b_{\ell_k})\cdot a_{i_k}$ \label{line:endSignFor2}
          	}
        	Set $S \gets \algCalcOpt((a_{i_1}'\mid\cdots\mid a_{i_r}'),(b_{\ell_1}',\cdots,b_{\ell_r}'))$. \label{line:compOpt2} \tcp{See Definition~\ref{defOptAxb} and Algorithm~\ref{calcOpt}.}
        	Set $X \gets X \cup S$ \label{line:addToX2}	
      	}
    }
    Set $Y \gets \br{(a_1,b_1),\cdots,(a_n,b_n)}$ \label{line:SetY}. \tcp{An arbitrary matching}
Set $S \gets \br{\left(x,\hat{\M}(Y,x,\cost)\right) \mid x\in X}$. \label{line:compOptMatch} \tcp{See Definition~\ref{def:optimalMatching}.}
Set $(\tilde{x},\tilde{\M}) \gets \displaystyle  \argmin_{(x,\M)\in S} \cost\left(Y_{\M},x\right)$

\Return $(\tilde{x},\tilde{\M})$ \label{line:compBestPair}
\end{algorithm}

The following theorem states that Algorithm~\ref{Alg_NoMatching} indeed returns an approximation to the optimal cost, even when the matching between rows of $A$ and $b$ is not given.
\begin{theorem} \label{costAxb_noMatch}
Let $A = (a_1 \mid \cdots \mid a_n)^T \in \REAL^{n\times d}$ be a matrix containing $n \geq d-1$ points in its rows, let $b = (b_1,\cdots,b_n)^T \in \REAL^n$, and let $Y = \br{(a_i,b_i)\mid i\in [n]}$. Let $D((a,\hat{b}),x) = |a^Tx-\hat{b}|$ for every $a\in\REAL^d$, $\hat{b} \in \REAL$ and $x\in \REAL^d$. Let $r$ be a scalar and $\cost$ be a function that satisfy Definition~\ref{def:cost} for $D$ and $f(v) = \norm{v}_1$. Let $(\tilde{x},\tilde{\M})$ be a pair of unit vector and permutation (matching function) which is the output of a call to $\matchingalgname(A,b,\cost)$; see Algorithm~\ref{Alg_NoMatching}. Then it holds that
\[
\cost(Y_{\tilde{\M}},\tilde{x}) \leq 4^{(d-1)r} \cdot \min_{x,\M} \cost(Y_{\M},x),
\]
where the minimum is over every unit vector $x\in \sphere^{d-1}$ and $\M\in\perms(n)$. Moreover, $(\tilde{x},\tilde{\M})$ is computed in $n^{O(d)}$ time.
\end{theorem}
\begin{proof}
See proof of Theorem~\ref{costAxb_noMatch_proof} in the appendix.
\end{proof}

\section{Coreset for Linear Regression} \label{sec:coreset}
In this section we provide a coreset for the constrained $\ell_p$ regression, which is a small improvement and generalization of previous results~\cite{jubran2018minimizing,dasgupta2009sampling,varadarajan2012sensitivity}

\subsection{Improvements via coreset}
An $\varepsilon$-coreset, which is a compression scheme for the data, is suggested in Theorem~\ref{theorem:coreset}. Streaming and distribution in near-logarithmic update time and space, including support for deletion of points in near-logarithmic time in $n$ (using linear space) is also supported by our algorithm when the matching between the rows of the matrix $A$ and the entries of $b$ is given; see Section~\ref{sec:lp_regression}. This is due to the fact that our coresets are composable, i.e., can be merged and reduced over time and computed independently on different subsets. This is now a standard technique, we refer the reader to ~\cite{bentley1978decomposable,agarwal2004approximating,feldman2011scalable,DBLP:journals/jmlr/LucicFKF17} for further details.

\begin{theorem} \label{theorem:coreset}
Let $d \geq 2$ be a constant integer. Let $A = (a_1 \mid \cdots \mid a_n)^T \in \REAL^{n\times d}$ be a matrix containing $n \geq d-1$ points in its rows, let $b = (b_1,\cdots,b_n)^T \in \REAL^n$, and let $w = (w_1,\cdots,w_n) \in [0,\infty)^n$.
Let $\varepsilon, \delta \in (0,1)$ and let $z\in [1,\infty)$.
Then in $O(n \log{n})$ time we can compute a weights vector $u = (u_1,\cdots,u_n)\in[0,\infty)^n$ that satisfies the following pair of properties.
\renewcommand{\labelenumi}{(\roman{enumi})}
\begin{enumerate}[noitemsep]
\item With probability at least $1-\delta$, for every $x\in \sphere^{d-1}$ it holds that
\[
\begin{split}
&(1-\varepsilon) \cdot \sum_{i\in [n]} w_i \cdot  |a_i^Tx-b_i|^z  \leq \sum_{i\in [n]}u_i\cdot |a_i^Tx-b_i|^z\\
\leq &(1+\varepsilon) \cdot \sum_{i\in [n]} w_i \cdot |a_i^Tx-b_i|^z.
\end{split}
\]
\item The weights vector $u$ has $O\left(\frac{\log{\frac{1}{\delta}}}{\varepsilon^2}\right)$ non-zero entries.
\end{enumerate}
\end{theorem}
\begin{proof}
See proof of Theorem~\ref{theorem:coreset_proof} in the appendix.
\end{proof}

\section{Experimental Results} \label{sec:ER}
To demonstrate the correctness and robustness of our algorithms, we implemented them in Matlab and compared them to some commercial software for non-convex optimization. A discussion is provided after each test, and an overall discussion is provided in Section~\ref{sec:overallDiscussion}. Open code is provided~\cite{openCode}.

\textbf{Hardware.} All the following tests were conducted using Matlab $R2019a$ and Maple $2018$ on a Lenovo W541 laptop with an Intel i7-4710MQ CPU @ 2.50GHZ and 8GB RAM.

Let $N(\mu,\sigma)$ denote a Gaussian distribution with mean $\mu$ and standard deviation $\sigma$ and let $U(r)$ denote a uniform distribution over $[0,r]$.

\begin{figure*}[t]
\centering
    \begin{subfigure}[h]{0.35\textwidth}
		\centering
		\includegraphics[width=\textwidth]{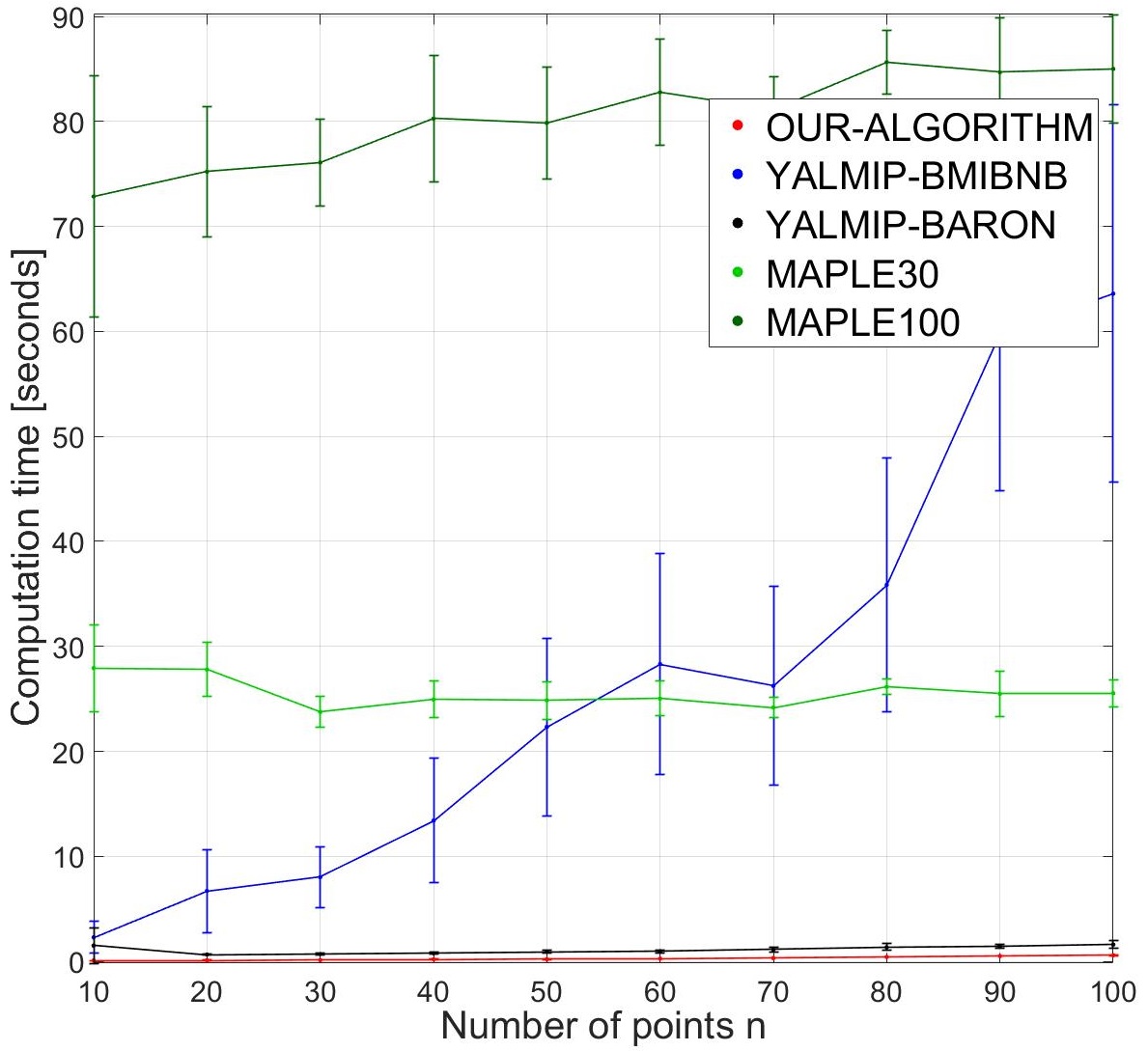} 
	\end{subfigure}
	\begin{subfigure}[h]{0.35\textwidth}
		\centering\centering		
		\includegraphics[width=\textwidth]{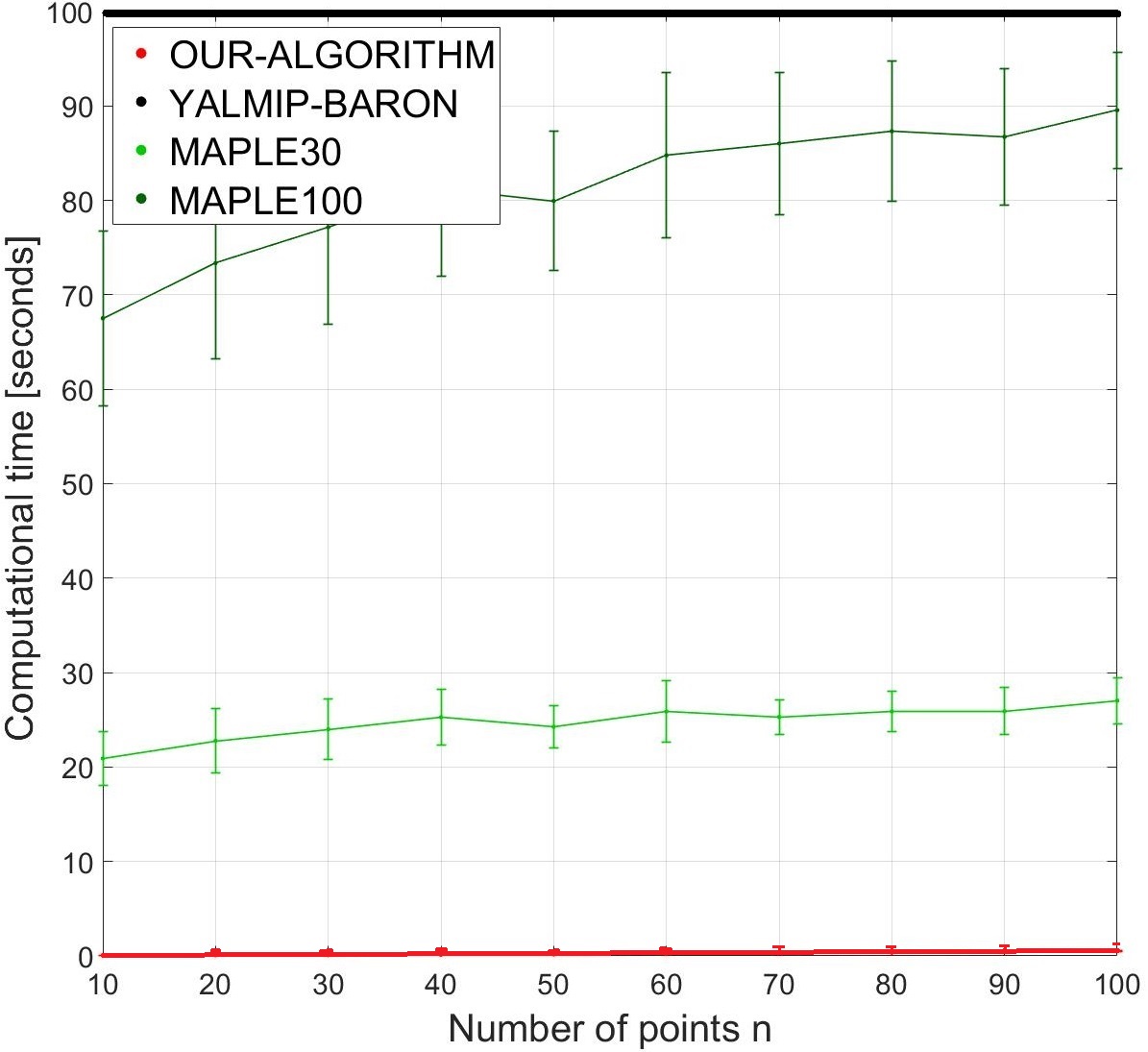} 
	\end{subfigure}
	\caption{Constrained $\ell_p$ regression; See details in Section~\ref{ERBasic}. Each test was repeated $15$ times. \textbf{(left)} Average time comparison and standard deviation for $p=1$ and $d=3$, \textbf{(right)} Average time comparison and standard deviation for $p=3.5$ and $d=3$. \Bnb{} did not support $p=3.5$. The value of the objective function in both cases was similar for all methods, hence we present only time comparison.}
	\label{ERfig:p1_p35}
\end{figure*}

\begin{figure*}[t]
  \centering
  \includegraphics[width=0.7\textwidth]{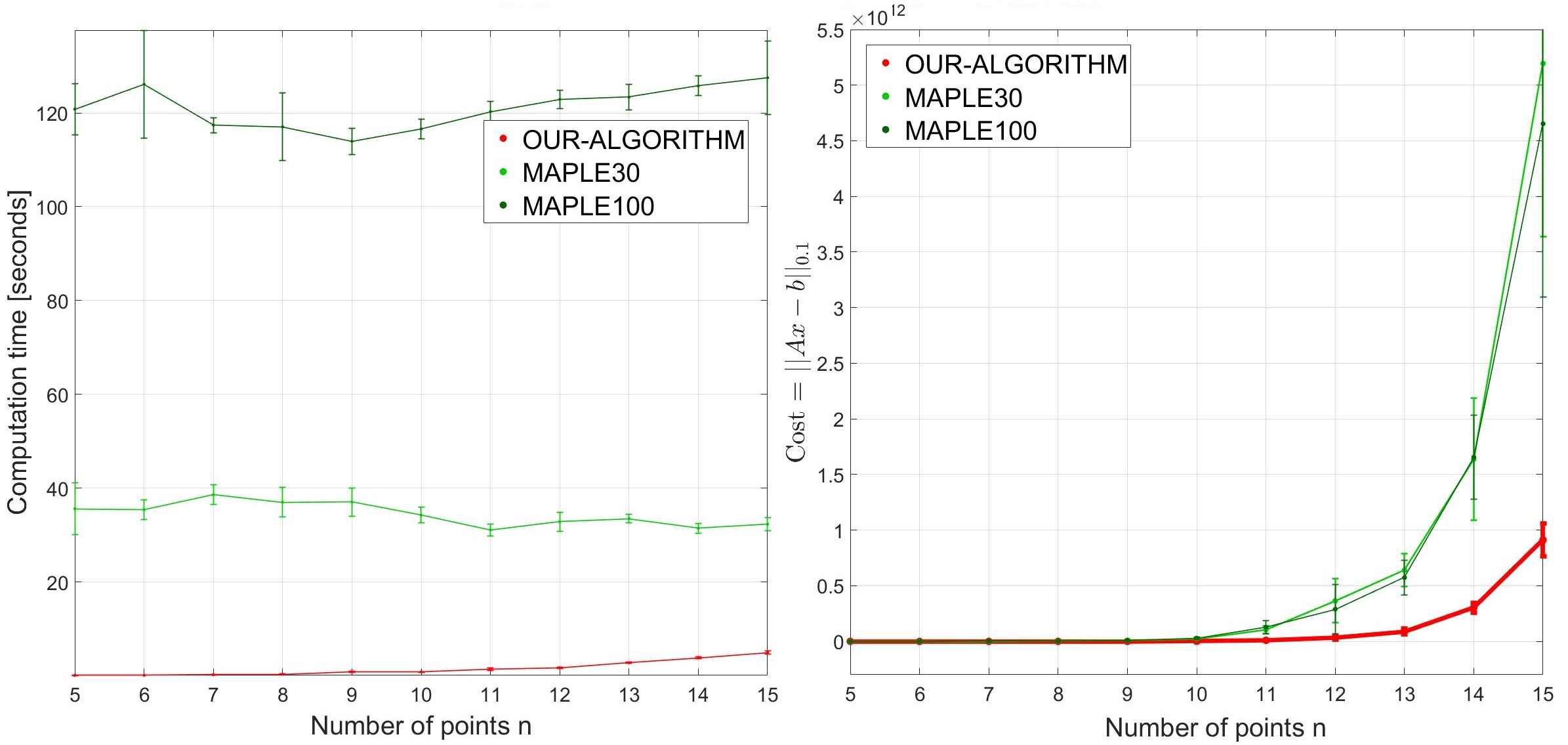}
  \caption{Constrained $\ell_p$ regression for $p=0.1$ and $d=5$; See details in Section~\ref{ERBasic}. The test was repeated $15$ times. \textbf{(left)} Average time and standard deviation comparison, \textbf{(right)} Average value of the objective function and standard deviation comparison. \Bnb{} and \Baron{} do not support $p\in(0,1)$. Also, it is worth mentioning that in many cases, \Maple{} and \MaplePro{} \textbf{did not} return a unit vector as output.}\label{ERfig:p01}
\end{figure*}

\begin{figure*}[t]
  \centering
  \includegraphics[width=0.7\textwidth]{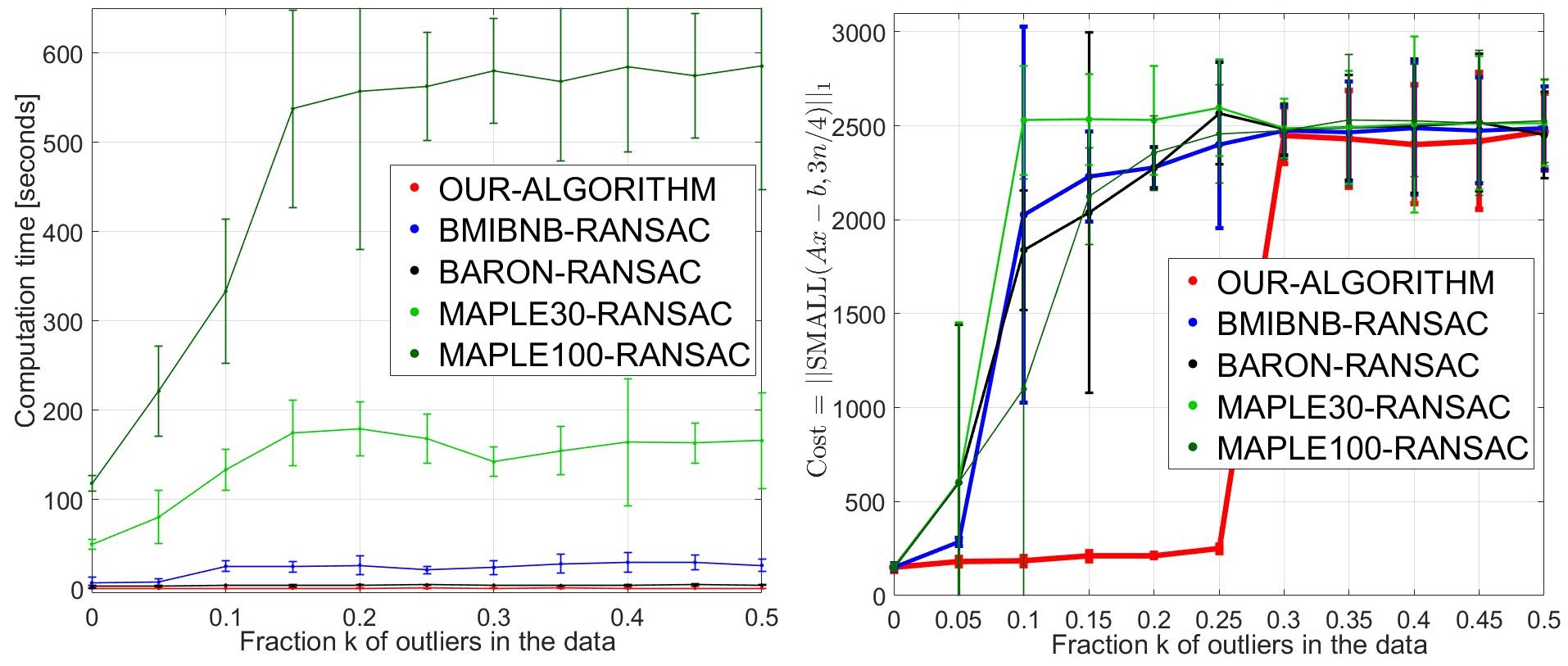}
  \caption{Constrained $\ell_p$ regression with outliers; See Section~\ref{ER:outliers} for details. The test was repeated $15$ times. \textbf{(left)} Average time and standard deviation comparison, \textbf{(right)} Average value of the objective function comparison, where the $x$-axis is the percentage $k$ of outliers in the data.}\label{ERfig:outliers}
\end{figure*}

\begin{figure*}[t]
  \centering
  \includegraphics[width=0.7\textwidth]{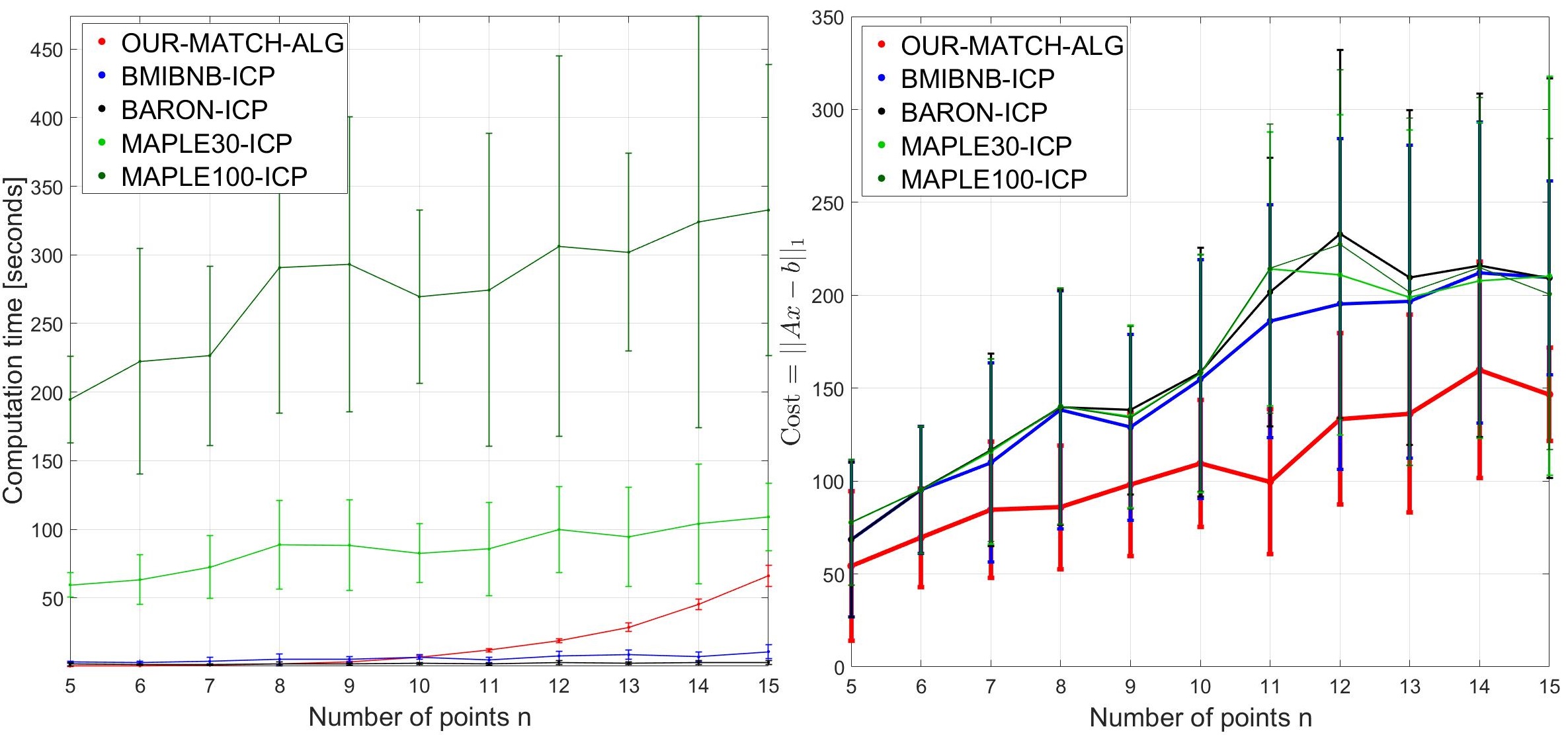}
  \caption{Constrained $\ell_p$ regression with unknown matching; See Section~\ref{ERNoMatching} for details. The test was repeated $15$ times. \textbf{(left)} Average time and standard deviation comparison, \textbf{(right)} Average value of the objective function and standard deviation comparison.}\label{ERfig:matching_withnoise}
\end{figure*}

\begin{figure*}[t]
  \centering
  \includegraphics[width=0.7\textwidth]{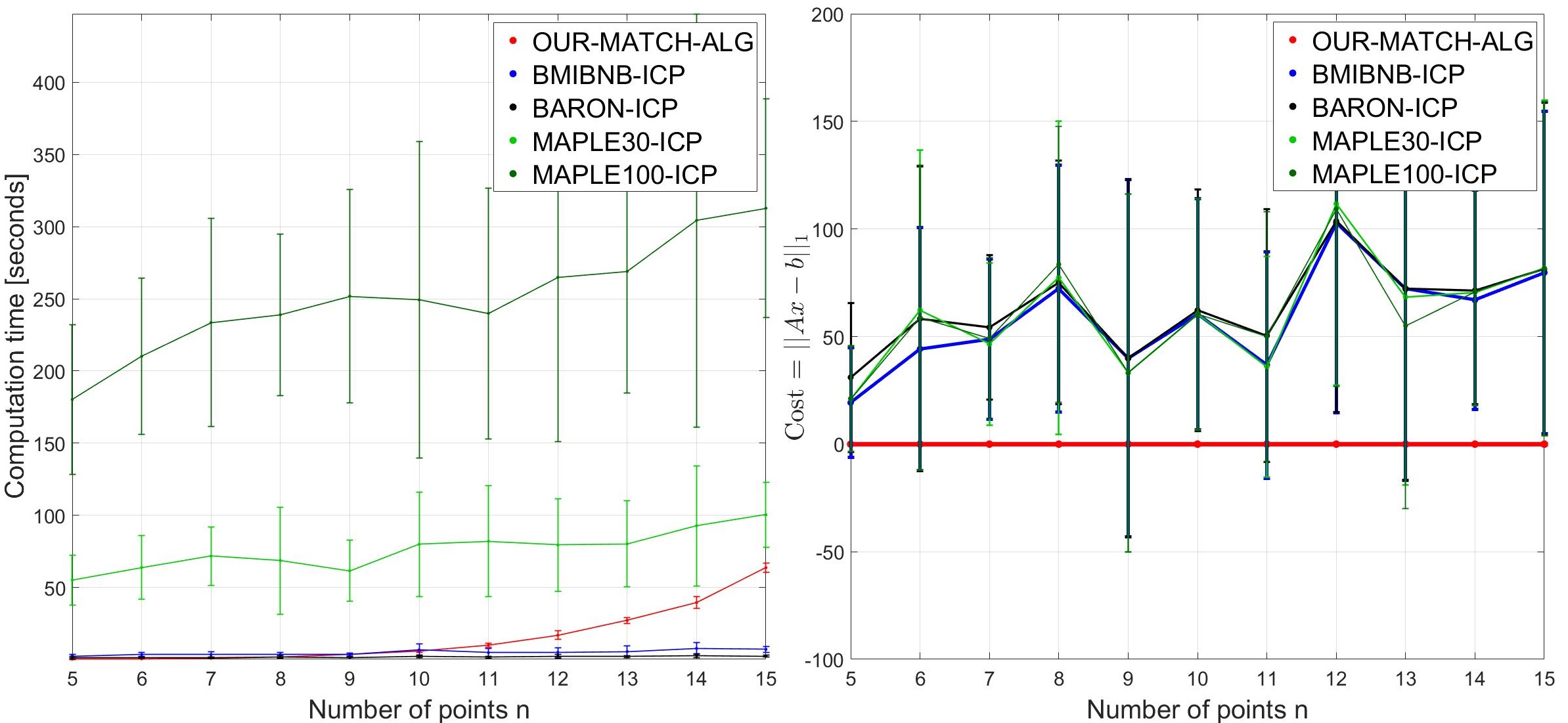}
  \caption{Constrained $\ell_p$ regression with unknown matching; See Section~\ref{ERNoMatching} for details. The test was repeated $15$ times. \textbf{(Left):} Average time and standard deviation comparison, \textbf{(Right):} Average value of the objective function comparison.}\label{ERfig:matching_nonoise}
\end{figure*}

\subsection{Constrained $\ell_p$ regression} \label{ERBasic}

\textbf{The dataset. }We conducted $3$ tests with different values of $p$ and $d$. The data for each test was a matrix $A\in\REAL^{n\times d}$ of $n$ points, and a vector $b\in\REAL^n$, where each entry of $A$ and $b$ was sampled from $U(200)$, i.e. uniformly at random from [0,200].

\textbf{Objective. }The test aims to solve problem~\eqref{mainEq3}.
The goal was to minimize the constrained $\ell_p$ regression function $f_p = \norm{Ax-b}_p$ for different values of $p\geq 1$.

\textbf{Algorithms. }We compare the following $5$ algorithms, which aim to minimize $f_p$ over every unit vector $x$.

\textbf{\OurAlg:} This algorithm computes $X = \algnameXCandidates(A,b)$, and returns $\min_{x\in X} f_p$.

\textbf{\Maple:} Maple~\cite{char2013maple} provides a function named $GlobalOptima$ from the $DirectSearch$ optimization package version 2~\cite{moiseev2011universal}. This function takes as input an objective function, a constraint, and an integer which specifies the number of initial ''simulated points``, where both the accuracy and the computation time increase with larger number of simulated points. The function aims to compute the global minima of non-linear multivariate functions under given constraints. We call this function from Matlab with the objective function $f_p$, the constraint $\norm{x}=1$, and the default number $(30)$ of initial simulated points.

\textbf{\MaplePro:} same as \Maple{} but with number=$100$ of initial simulated points.

\textbf{\Bnb:} The Yalmip library~\cite{lofberg2004yalmip} for Matlab provides a function named $optimize$, which takes as input an objective function to minimize, a constraint, and a solver program, and aims to optimize the given function under the given constraint using the given solver.
We run this function with the objective function $f_p$, the constraint $\norm{x}=1$, and the $bmibnb$ solver~\cite{narendra1977branch} which aims to solve non-convex problems.

\textbf{\Baron:} same as the previous algorithm, but with the $Baron$ solver~\cite{ts:05,sahinidis:baron:17.8.9} which is a widely used commercial software for global optimization.

\textbf{The results }are shown in Figures~\ref{ERfig:p1_p35} and~\ref{ERfig:p01}.

\textbf{Discussion. }Figure~\ref{ERfig:p1_p35} presents the computation time required for the suggested methods with $p\geq 1$. We conducted two tests: $p=1$ and $p=3.5$, both for increasing values of $n$.

As the graph for $p=1$ shows, \OurAlg{} managed to compute the output in a fraction of the time it took other methods to compute their output, except for \Baron{}, which was a close second. In this test, the cost (error) of all methods was similar.

The graph for $p=3.5$ presents similar results. However, this time \Baron{} took time which is more than $10$ folds larger than the other methods. Hence it is presented as a thick black line at the top of the graph. Also for this case the cost (error) of all methods was similar.

Figure~\ref{ERfig:p01} presents both the cost (error) and the computation time of the suggested methods for $p=0.1$ with increasing values of $n$. In this test, we were only able to run \Maple{} and \MaplePro{}, as the other methods either did not support $p<1$ or did not return an output in a reasonable amount of time.

As the time comparison graph shows, \OurAlg{} required time which is $2$ orders of magnitude smaller than the time it took \Maple{} and \MaplePro{}.

The costs graph shows that at $n=10$, the cost of \Maple{} and \MaplePro{} start to grow exponentially faster than the cost of \OurAlg. The gap between the cost of \OurAlg and the cost of the other methods continues to grow with $n$.

\subsection{Robustness to outliers} \label{ER:outliers}
The test in this subsection aims to solve Problem~\eqref{mainEq5}.

\textbf{The dataset. }We generated a matrix $A \in \REAL^{80\times 3}$ where each entry was sampled randomly from $U(200)$, and a unit vector $x\in\REAL^3$ was randomly generated. We then define $b = Ax$ . Gaussian noise drawn from $N(0,30)$ was then added to each of the entries of $A$ and $b$ respectively.
We then added noise sampled randomly from $U(20000)$ to $\floor{k\cdot A}$ entries of $A$ and their corresponding entries of $b$, where $k\in [0,1/2]$.

\paragraph{Objective. }This test aims to solve~\eqref{mainEq5}.
The goal of the experiment was to minimize the $f_O(A,b,x) = \sum_{i\in [n]} \norm{\smallest(Ax-b,3n/4)}_1$ over every unit vector $x$.

\vspace{-0.4cm}
\paragraph{Algorithms. }We compared the following $4$ algorithms, which aim to minimize $f_O(A,b,x)$ over every unit vector $x$.

\textbf{\OurAlg{}} above with objective function $f_O(A,b,x)$.

\textbf{\MapleRansac:} A random sample consensus (RANSAC) scheme~\cite{fischler1981random}, where at each iteration we sample uniformly at random corresponding subsets $\hat{A}$, and $\hat{b}$ of size $3$ from $A$ and $b$ respectively, called \Maple{} with objective function $\norm{\hat{A}x-\hat{b}}_1$ and constraint $\norm{x}=1$ to obtain some unit vector $\hat{x}\in \REAL^3$, computed the number of pairs $(a_i,b_i)$ that satisfy $|a_i^T\hat{x}-b_i| \leq 200$ (inlier points) and pick the unit vector with maximal number of inliers denoted by $A^*$ and $b^*$. Final call is made to \Maple{} with objective function $\norm{A^*x-b^*}_1$ and constrain $\norm{x}=1$ that returns its output.

\textbf{\BnbRansac:} As \MapleRansac{} but calls \Bnb{} instead of \Maple{}.

\textbf{\BaronRansac:} As \MapleRansac{} but calls \Baron{} instead of \Maple{}.

\textbf{The results }are shown in Figure~\ref{ERfig:outliers}.

\textbf{Discussion. }Figure~\ref{ERfig:outliers} presents both the cost (error) and the computation time of the suggested methods, while increasing the value of $k$ (the fraction of outliers in the data).

As the time comparison graph shows, \OurAlg{} managed to compute the output faster than the other methods.

The costs graph demonstrates the robustness in practice of our algorithm to outliers. The cost of \OurAlg{} was roughly constant as the fraction $k$ of outliers grew until $k=0.25$. At $k>0.25$, the cost of \OurAlg{} had a sharp increase. This is not surprising since the cost function sums over $75\%$ of the points with the smallest cost. In other words, when $k>0.25$, $75\%$ of the data with the smallest cost must contain points who are outliers.
The other methods showed an immediate increase in cost as soon as outliers were introduced to the data ($k>0$).

\subsection{Unknown correspondences between $A$ and $b$} \label{ERNoMatching}
The test in this subsection aims to solve~\eqref{mainEq4}.

\textbf{The dataset }is a matrix $A \in \REAL^{n\times 3}$ where each entry was sampled randomly from $U(200)$, and $Ax=b$ for some random unit vector $x\in\REAL^d$. The rows of $A$ have then been shuffled with a random permutation. We then added small Gaussian noise drawn from $N(0,10)$ to every entry of $A$ and $b$ respectively.
\vspace{-0.1cm}

\textbf{Objective. }Let $Y=\br{(a_i,b_i)}_{i=1}^n$. The goal of the experiment was to minimize $f_M(Y_{\M},x) = \sum_{i\in [n]} |a_i^Tx-b_{\M(i)}|$ over every unit vector $x\in\REAL^3$ and permutation $\M$.

\paragraph{Algorithms. }We compared the $5$ following algorithms.

\textbf{\OurMatchingAlg: }Run $\matchingalgname(A,b,f_M)$; See Algorithm~\ref{Alg_NoMatching}.

\textbf{\MapleICP: }An iterative closest point (ICP) scheme where we alternate between the following two steps until there is no sufficient change in the cost function: (1) Call \Maple{} with objective function $\norm{A'x-b}_1$ and constraint $\norm{x}=1$ to obtain a unit vector $x'$, where $A'=A$ at the first iteration, and (2) Compute the optimal permutation (a shuffling of A') that minimizes the differences between $A'x'$ and $b$, and apply the shuffling to $A'$. This is done using the Hungarian method~\cite{kuhn1955hungarian}.

\textbf{\MapleProICP: }Similar to \MapleICP{} but calls \MaplePro{} instead of \Maple{}.

\textbf{\BnbICP: }Similar to \MapleICP{} but calls \Bnb{} instead of \Maple{}.

\textbf{\BaronICP: }Similar to \MapleICP{} but calls \Baron{} instead of \Maple{}.

\textbf{The results }are shown in Figure~\ref{ERfig:matching_withnoise}.

\textbf{Discussion. }As the costs graph in Figure~\ref{ERfig:matching_withnoise} shows, \OurMatchingAlg{} solved~\eqref{mainEq5} with higher accuracy than all other methods.

\subsubsection{Unknown correspondences between $A$ and $b$, without noise. }
We conducted another test similar to the previous test, for the case where the correspondences are unknown, but without adding noise to the data. In other words, a matrix $A \in \REAL^{n\times 3}$ was generated, where each entry was sampled randomly from $U(200)$, and a corresponding vector $b\in \REAL^n$ was computed such that $Ax=b$ for some random unit vector $x\in\REAL^d$. The rows of $A$ have then been shuffled with a random permutation. We \textbf{did not} add noise to the data. The objective function and the algorithms tested are similar to the previous test.

\textbf{The results }are shown in Figure~\ref{ERfig:matching_nonoise}.

\textbf{Discussion. }As expected due to the constant factor approximation guaranteed of all our algorithms, the cost of \OurMatchingAlg{} was constant at $0$. Which means \OurMatchingAlg was able to recover the unit vector and the matching function correctly, while the other methods could not.

\subsection{Overall discussion} \label{sec:overallDiscussion}
\paragraph{Variance. }As all the figures in Section~\ref{sec:ER} show, the variance of our algorithms was smaller and more stable than the variance of the other methods. This happens since our algorithms are guaranteed to compute an approximated ("good") result. In other words, the output cost of our algorithms can not have a dignificant change, since it is guaranteed to be with in small range of the optimal objective function value. Hence the small variance. The running time of our algorithms is deterministic. Hence the computation time variance is small as well.

\paragraph{The constants behind our approximation factors. }
Our theorems guarantee that the output cost of our algorithms is always within range of the optimal value of the objective function used, up to a constant factor of $4^{d-1} \in O(4^{d})$. This bound on the approximation constant is a worst case analysis bound. In practice however, this constant is much smaller.

\paragraph{Decreasing the computation time of our algorithms.}
Observe that since our algorithms are embarrassingly parallel, using a computer with $M$ cores would have reduced the running time of our algorithms by a factor of $M$.

\section{Conclusion and Open Problems} \label{sec:conclude}
We proposed the first provable polynomial time approximation algorithms for the constrained $\ell_p$ regression problem, for any constant $p>0$ and $d\geq1$, including versions for handling outliers, and unknown order of rows in $A$. Using coresets, the running time is near linear in some cases. Experimental results show that our algorithms outperform existing commercial solvers. Open problems: (i) Running time that is polynomial in $d$ is hopeless since the problem is NP-hard, but additive approximations may be obtained via projection on random subspaces or PCA, (ii) $(1+\eps)$-approximations, (iii) near-linear time algorithms for the unknown matching case, and streaming version for these cases.

\bibliography{references}
\bibliographystyle{icml2019}


\newpage
\clearpage
\appendix

\section{Regression with Given Matching}

The following corollary states that if we double the distance on a unit sphere to a point $y$, which is the closest point on the unit sphere to a line $\ell$, then the distance to $\ell$ will grow by at most a multiplicative factor of $4$.
\begin{corollary} [Corollary~\ref{cor4ApproxVectors}] \label{cor4ApproxVectors_proof}
Let $a \in \REAL^2\setminus \br{0}$ and $b \geq 0$. Let $y \in \argmin_{x\in \sphere^1} |a^T x - b|$. Then for every $u^*,u' \in \sphere$ such that $\norm{u'-y} \leq 2\cdot \norm{u^*-y}$ we have
\[
|a^Tu'-b| \leq 4\cdot |a^Tu^*-b|.
\]
\end{corollary}
\begin{proof}
Let $g:[0,2\pi) \to [0,\infty)$ such that $g(\alpha) = \left|\sin(\alpha) - \frac{b}{\norm{a}}\right|$.
Let $M(g) = \argmin_{\alpha \in [0,2\pi)}g(\alpha)$ be the set of minima of $g$.

By replacing $b$ with $\frac{b}{\norm{a}}$ in Corollary~\ref{corCosLip}, $g$ is a piecewise $2$-log-Lipschitz function. Hence, by Definition~\ref{def:PieceLip}, there is a partition $X_1,\cdots,X_m$ of $[0,2\pi)$, a set of $2$-log-Lipschitz functions $h_1,\cdots,h_m$, and a set $M(g) = \br{\alpha^*_1,\cdots,\alpha^*_m}$ of minima such that for every $i\in [m]$ and $\alpha_1,\alpha_2 \in X_i$ where $|\alpha_2-x_i| \leq 2\cdot |\alpha_1-_i|$
we have that
\begin{equation} \label{mainGProp}
g(\alpha_2) = h_i(|\alpha_2-x_i|) \leq 4\cdot h_i(|\alpha_1-x_i| =4\cdot g(\alpha_1),
\end{equation}
where the first derivation holds by property (iii) of Definition~\ref{def:PieceLip}, the second derivation holds by combining that $h_i$ is a $2$-lop-Lipschitz function with $|\alpha_2-x_i| \leq 2\cdot |\alpha_1-x_i|$, and the last derivation holds by Property (iii) of Definition~\ref{def:PieceLip}.


Without loss of generality, assume that $\frac{a}{\norm{a}} = (1,0)^T$. Otherwise rotate the coordinates system.
Let $x:\REAL\to\REAL^2$ such that $x(\alpha) = (\sin{\alpha},-\cos{\alpha})^T$.
For every $\alpha \in [0,2\pi)$ we now have that
\begin{equation} \label{axbToG}
\begin{split}
|a^Tx(\alpha) - b|& = \norm{a}\left|\frac{a^T}{\norm{a}}x(\alpha)-\frac{b}{\norm{a}}\right|\\ &= \norm{a}\left|\sin{\alpha}-\frac{b}{\norm{a}}\right| = \norm{a}\cdot g(\alpha).
\end{split}
\end{equation}

Let $\alpha_y \in [0,2\pi)$ such that $y = x(\alpha_y)$. We now have that
\[
\begin{split}
\norm{a} g(\alpha_y) & = |a^Ty-b| =\min_{\alpha\in[0,2\pi)}|a^Tx(\alpha)-b|\\ &= \min_{\alpha\in[0,2\pi)} \norm{a} g(\alpha),
\end{split}
\]
where the first derivation is by combining~\eqref{axbToG} and the definitions of $y$ and $\alpha_y$, the second derivation is by the definition of $x(\alpha)$, and the last derivation is by~\eqref{axbToG}.
Hence, $\alpha_y \in M(g)$.

Let $u',u^* \in \sphere$ and $\alpha',\alpha^* \in [0,2\pi)$ such that $u' = x(\alpha')$ and $u^*=x(\alpha^*)$. Since $\norm{u'-y} \leq 2\cdot \norm{u^*-y}$ by the assumption of the Corollary, we get that $|\alpha'-\alpha_y| \leq 2 |\alpha^*-\alpha_y|$ by simple linear algebra.
Due to the last inequality and since $\alpha_y \in M(g)$, we can substitute $\alpha_2 = \alpha'$ and $\alpha_1 = \alpha^*$ in~\eqref{mainGProp} to obtain
\begin{equation}\label{eqg4approx}
g(\alpha') \leq 4 g(\alpha^*).
\end{equation}
Corollary~\ref{cor4ApproxVectors} now holds as
\[
\begin{split}
|a^Tu'-b| & = |a^Tx(\alpha')-b|=\norm{a}\cdot g(\alpha')\\ &\leq 4\norm{a}\cdot g(\alpha^*) = 4|a^Tx(\alpha^*)-b|\\ &=4|a^Tu^*-b|,
\end{split}
\]
where the first derivation is by the definition of $u'$, the second derivation is by substituting $\alpha=\alpha'$ in~\eqref{axbToG}, the third derivation is by~\eqref{eqg4approx}, the fourth derivation is by substituting $\alpha=\alpha^*$ in~\eqref{axbToG} and the last derivation is by the definition of $\alpha^*$.
\end{proof}

\begin{lemma} [Lemma~\ref{corAxb4Approx}]\label{corAxb4Approx_proof}
Let $A = (a_1 \mid \cdots \mid a_n)^T \in \REAL^{n\times d} \setminus \br{0}^{d\times d}$ be a non-zero matrix such that $n \geq d-1 \geq 1$ points, and let $b = (b_1,\cdots,b_n)^T \in [0,\infty)^n$ and $x^* \in \sphere^{d-1}$. Then there exists $j\in [n]$ where $\norm{a_j} \neq 0$ and $x_j \in \argmin_{x\in \sphere^{d-1}} |a_j^Tx-b_j|$ such that for every $i\in [n]$
\[
|a_i^Tx_j-b_i| \leq 4\cdot |a_i^Tx^*-b_i|.
\]
\end{lemma}
\begin{proof}
For every $i\in [n]$, if $\norm{a_i} = 0$, then for every $x \in \sphere^{d-1}$ it holds that $|a_i^Tx-b_i| = |b_i| \leq 4 |b_i| = 4 |a_i^Tx^*-b_i|$.
Therefore, we ignore vectors of zero length and assume that $A$ does not contain any zero row.

For every $i\in [n]$, let
\[
M_i = \argmin_{x\in \sphere^{d-1}} \left|\frac{a_i^T}{\norm{a_i}} x -\frac{b_i}{\norm{a_i}}\right|,
\]
\[
x' \in \argmin_{x \in \bigcup_{i\in [n]} M_i} \norm{x^* - x},
\]
\[
j \in [n] \text{ such that }x' \in M_j,
\]
i.e., $M_i$ is the set of all unit vectors that minimize $|a_i^Tx-b_i|$, $x'$ is the unit vector that is closest to $x^*$ among all vectors in $\bigcup_{i\in [n]} M_i$, and $j$ is the index of the set $M_j$ such that $x' \in M_j$.

Put $i\in [n]$. We prove that
\begin{equation} \label{eqToProve}
|a_i^Tx'-b_i| \leq 4\cdot |a_i^Tx^*-b_i|.
\end{equation}
This would prove the lemma since $x' \in M_j = \argmin_{x\in \sphere^{d-1}} |a_j^Tx-b_j|$.

Indeed, assume without loss of generality that $\frac{a_i}{\norm{a_i}} = (1,0,\cdots,0)^T$ and that $x^* = (x^*_1,x^*_2,0,\cdots,0)^T$, $x^*_2 \geq 0$. Otherwise, rotate the coordinate system as follows: rotate $\frac{a_i}{\norm{a_i}}$ until it coincides with the $x$-axis. Then, rotate the system around the $x$-axis (i.e., without changing $\frac{a_i}{\norm{a_i}}$), until $x^*$ intersects the $xy$-plane at the halfspace with positive $y$ values.

We prove~\eqref{eqToProve} via the following case analysis: \textbf{(i) }$b_i \leq \norm{a_i}$ and \textbf{(ii) }$b_i>\norm{a_i}$.

\textbf{Case (i): }$b_i \leq \norm{a_i}$. Let $y = (y_1,y_2,\cdots,y_d)^T := (\frac{b_i}{\norm{a_i}},\sqrt{1-\frac{b_i^2}{\norm{a_i}^2}},0,\cdots,0)^T$. Observe that $y \in M_i$ since
\[
\begin{split}
|a_i^Ty-b_i| & = \norm{a_i} \left|\frac{a_i^T}{\norm{a_i}}y - \frac{b_i}{\norm{a_i}}\right|\\
& = \norm{a_i} \left|\frac{b_i}{\norm{a_i}} - \frac{b_i}{\norm{a_i}}\right| = 0.
\end{split}
\]

Identify $x' = (x'_d1,\cdots,x'_d)$ and $x^* = (x^*_1,\cdots,x^*_d)$.
Let $v' = \left(x'_1, \sqrt{1-{x'_1}^2}\right)^T$, $v^* = \left(x^*_1,\sqrt{1-{x^*_1}^2}\right)^T$ and $y' = (y_1,y_2)^T = \left(\frac{b_i}{\norm{a_i}},\sqrt{1-\frac{b_i^2}{\norm{a_i}^2}}\right)^T$.
Since $v'_1 = x'_1$ and $\norm{v'} = \norm{x'}=1$ we obtain
\begin{equation}\label{eqv2bigger}
\begin{split}
{v'_2}^2 & =\norm{v'}^2-v_1^2=\norm{v'}^2-{x'_1}^2\\
&=\norm{x'}^2-{x'_1}^2 =\sum_{j=2}^d {x'_j}^2 \geq {x'_2}^2
\end{split}
\end{equation}
We now have that
\begin{align}
\norm{x'-x^*}^2 & = \norm{x'}^2 + \norm{x^*}^2 -2{x'}^Tx^* \nonumber\\
& = \norm{v'}^2 + \norm{v^*}^2 -2(x'_1\cdot x^*_1 + x'_2\cdot x^*_2) \label{secondEqua}\\
& = \norm{v'}^2 + \norm{v^*}^2 -2(v'_1\cdot v^*_1 + x'_2\cdot x^*_2) \label{thirdEqua}\\
& \geq \norm{v'}^2 + \norm{v^*}^2 -2(v'_1\cdot v^*_1 + v'_2\cdot x^*_2) \label{firstInequa}\\
& = \norm{v'-v^*}^2 \label{eqx2tov2},
\end{align}
where~\eqref{secondEqua} holds since only the first two entries of $x^*$ are non-zero, \eqref{thirdEqua} is by the definitions of $v'$ and $v^*$, and~\eqref{firstInequa} holds by combining the squared root of~\eqref{eqv2bigger} with the assumption that $x^*_2 \geq 0$.

Hence, we obtain that
\begin{align}
\norm{v'-v^*} & \leq \norm{x'-x^*} \label{eqdists1}\\
& \leq \norm{y-x^*} \label{eqdists2}\\
& = \norm{y'-v^*} \label{eqdists3},
\end{align}
where~\eqref{eqdists1} is by taking the squared root of~\eqref{eqx2tov2}, \eqref{eqdists2} holds by combining the definition of $x'$ and the fact that $y \in M_i \subseteq \bigcup_{l\in [n]} M_l$, and~\eqref{eqdists3} holds by the definitions of $y'$ and $v^*$.
Hence,
\[
\norm{v'-y'} \leq \norm{v'-v^*} + \norm{v^*-y'} \leq 2\cdot \norm{v^*-y'},
\]
where the first inequality is the triangle inequality, and the second inequality is by~\eqref{eqdists3}.

Since $|(1,0)y'-\frac{b_i}{\norm{a_i}}| = 0$ we have $y' \in \argmin_{x\in \sphere} |(1,0)x - \frac{b_i}{\norm{a_i}}|$.
Therefore, substituting $a = (1,0)^T$, $b=\frac{b_i}{\norm{a_i}}$, $y=y'$, $u^* = v^*$ and $u' = v'$ in Corollary~\ref{cor4ApproxVectors} yields that
\begin{equation}\label{eqUsingCor}
\left|(1,0)v'-\frac{b_i}{\norm{a_i}}\right| \leq 4\cdot \left|(1,0)v^*-\frac{b_i}{\norm{a_i}}\right|.
\end{equation}

Hence,
\[
\begin{split}
|a_i^Tx'-b_i| & = \norm{a_i} \cdot \left|\frac{a_i^T}{\norm{a_i}}x'-\frac{b_i}{\norm{a_i}}\right|\\
& = \norm{a_i} \cdot \left|(1,0)v'-\frac{b_i}{\norm{a_i}}\right|\\
& \leq 4\norm{a_i} \cdot \left|(1,0)v^* -\frac{b_i}{\norm{a_i}}\right|
= 4\cdot |a_i^Tx^* - b_i|,
\end{split}
\]
where in the first equality we multiply and divide by $\norm{a_i}$, the second and last equality are by the assumption $\frac{a_i}{\norm{a_i}}=(1,\ldots,0)^T$, and the inequality is by~\eqref{eqUsingCor}.

\textbf{Case (ii): }$b_i>\norm{a_i}$. Let $a = \frac{a_i}{\norm{a_i}}$ and $b = \frac{b_i}{\norm{a_i}}$.
In this case, we have
\begin{equation} \label{eqReduceToCasei}
\begin{split}
|a_i^Tx'-b_i| & = \norm{a_i} \cdot \left|b-\frac{a_i^T}{\norm{a_i}}x'\right|\\
&= \norm{a_i} \cdot \left|b-1+1-\frac{a_i^T}{\norm{a_i}}x'\right|\\
&= \norm{a_i} \cdot \left(b-1 + \left|1-a^Tx'\right|\right),
\end{split}
\end{equation}
where the last derivation holds since $b-1>0$ by the assumption of Case (ii), and $1\geq x'_1 = \frac{a_i^T}{\norm{a_i}}x' = a^Tx'$ since $x'$ is a unit vector.

Since for $b_i=1$ and $a_i = a$ the condition $b_i \leq \norm{a_i}$ of Case (i) holds, we obtain by Case (i) that
\begin{equation}\label{eqUseCasei}
|1-a^Tx'| = |a^Tx'-b_i| \leq 4|a^Tx^*-b_i| = 4 |a^Tx^*-1|.
\end{equation}

This proves Case (ii) as
\begin{align}
|a_i^Tx'-b_i| & = \norm{a_i} \cdot \left(b-1 + \left|1-a^Tx'\right|\right) \label{eq0case2}\\
& \leq \norm{a_i} \cdot \left(b-1 + 4\cdot\left|1-a^Tx^*\right|\right) \label{eq1case2}\\
& \leq 4\norm{a_i} \cdot \left(b-1 + \left|1-a^Tx^*\right|\right) \label{eq2case2}\\
& = 4\norm{a_i}\cdot \left|b-a^Tx^*\right| \label{eq3case2}\\
& = 4\cdot |a_i^Tx^*-b_i| \nonumber,
\end{align}
where~\eqref{eq0case2} is by~\eqref{eqReduceToCasei}, \eqref{eq1case2} is by~\eqref{eqUseCasei}, \eqref{eq2case2} holds since $b-1 \geq 0$, and~\eqref{eq3case2} holds since $a^Tx^* \leq 1 \leq b$.
\end{proof}

As explained in Section~\ref{sec:lp_regression}, the optimization problem $\min_{x\in\sphere^{d-1}}\norm{Ax-b}_1$ can be interpreted as computing a point $x'$ on the unit sphere that minimizes the sum of distances to $n$ given hyperplanes $H = \br{h_1,\cdots,h_n}$ in $\REAL^d$.
The following lemma suggests that there are $m\leq d-1$ hyperplanes from $H$ and a unit vector $x'$ in the intersection in the intersection of the first $m-1$ hyperplanes, that is closest to the last hyperplane, that is closer to every one of the $n$ hyperplane in $H$, up to a factor of $4^{d-1}$ than $x^*$, i.e., $|a_i^Tx'-b_i| \leq 4^{d-1}\cdot|a_i^Tx^*-b_i|$ for every $i\in [n]$.

To find such a vector $x'$, we begin with an initial set of candidate solutions $\sphere^{d-1}$, i.e., all unit vectors in $\REAL^d$.
The proof of the following lemma consists of at most $d-1$ steps, each step adds $c\geq 1$ more constraints on the candidate unit vectors for the solutions, and adds another factor of $4$ to the final approximation factor. This induction terminates when the number of added constraints is $d-1$ (so there is a finite set of candidates for $x'$).
Each step suggests that we can rotate $x'$, until it minimizes its distances $|a_j^Tx'-b_j|$ to one of the hyperplanes (the $j$th hyperplane in the proof), without increasing each of the other $n-1$ distances by more than a multiplicative factor of $4$. This follows from Lemma~\ref{corAxb4Approx}.
Afterwards, we reduce the problem to an instance of the same optimization problem but with less free parameters.
\begin{lemma} [Lemma~\ref{lemApproxD}] \label{lemApproxD_proof}
Let $A = (a_1 \mid \cdots \mid a_n)^T \in \REAL^{n\times d}$ be a non-zero matrix of $n \geq d-1 \geq 1$ rows, let $b = (b_1,\cdots,b_n)^T \in [0,\infty)^n$, and let $x^* \in \sphere^{d-1}$.
Then there is a set $\br{{i_1},\cdots,{i_{r}}} \subseteq [n]$ of $r \in [d-1]$ indices such that for $X = \opt((a_{i_1}\mid\cdots\mid a_{i_{r}}), (b_{i_1},\cdots,b_{i_{r}}))$ and every $i\in [n]$, there is $x'\in X$ that satisfies
\begin{equation} \label{mainToProve_proof}
|a_i^Tx'-b_i| \leq 4^{d-1} \cdot |a_i^Tx^*-b_i|.
\end{equation}
Moreover, \eqref{mainToProve_proof} holds for every $x'\in X$ if $|X| = \infty$.
\end{lemma}
\begin{proof}
For every $i\in [n]$, if $\norm{a_i} = 0$, then for every $x \in \sphere^{d-1}$ we have $|a_i^Tx-b_i| = |b_i| \leq 4\cdot |b_i| = 4\cdot |a_i^Tx^*-b_i|$ and~\eqref{mainToProve_proof} trivially holds.
Therefore, in the rest of the proof we assume $\norm{a_i} \neq 0$ for every $i\in [n]$.

The proof is by induction on the dimension $d$.

\textbf{Base case for $d=2$: }Substituting $A$, $b$, and $d=2$ in Lemma~\ref{corAxb4Approx} yields that there is $j\in [n]$ and $x_j\in \argmin_{x\in \sphere}|a_j^Tx-b_j|$ that satisfy, for every $i\in [n]$,
\[
|a_i^Tx_j-b_i| \leq 4\cdot |a_i^Tx^*-b_i|.
\]
For every $a\in \REAL^2$ and $b \geq 0$, we have that
\begin{equation} \label{opt2D}
\begin{split}
& \argmin_{x\in \sphere}|a^Tx-b|\\ &= \begin{cases}
  \br{\frac{a}{\norm{a}}}, & \mbox{if } b \geq \norm{a} \\
  \br{x\in \sphere \mid \frac{a^T}{\norm{a}}x = \frac{b}{\norm{a}}}, & \mbox{otherwise};
\end{cases}
\end{split}
\end{equation}
see Figure~\ref{Alg2DCase} for a geometric illustration of the solutions for $\argmin_{x\in \sphere}|a^Tx-b|$, in both cases where $b \geq \norm{a}$ and $b < \norm{a}$.
Substituting $m=1$, $a_1 = a_j$ and $b_1=b_j$ in Definition~\ref{defOptAxb} implies $x_j \in \argmin_{x\in\sphere} |a_j^Tx-b_j| = \opt(x_j,b_j)$. By~\eqref{opt2D}, $|\opt(x_j,b_j)| = \left|\argmin_{x\in\sphere} |a_j^Tx-b_j|\right| \leq 2$.
Hence, $X = \opt(x_j,b_j)$ and $x' = x_j$ satisfies~\eqref{mainToProve_proof} in Lemma~\ref{lemApproxD}.

\textbf{Case $d\geq 3$:} Inductively assume that Lemma~\ref{lemApproxD} holds for $d'=d-1$.
For every $i\in [n]$, let $\opt_i = \opt(a_i,b_i)$; See Definition~\ref{defOptAxb}. Put $i\in [n]$.
If $\norm{a_{i}} \leq b_i$, then
\begin{equation} \label{optCasei}
\begin{split}
\opt_i & = \argmin_{x\in \sphere^{d-1}}|a_i^Tx-b_i| = \argmin_{x\in \sphere^{d-1}} \left(b_i - a_i^Tx\right)\\ &= \argmax_{x\in \sphere^{d-1}} a_i^Tx = \br{\frac{a_{i}}{\norm{a_{i}}}}.
\end{split}
\end{equation}
If $\norm{a_{i}} > b_i$, then $\min_{x\in \sphere^{d-1}}|a_i^Tx-b_i| = 0$. In this case,
\begin{equation} \label{optCaseii}
\begin{split}
\opt_i & = \argmin_{x\in \sphere^{d-1}}|a_i^Tx-b_i| = \br{x \in \sphere^{d-1} \mid a_i^Tx=b_i}\\ &= \br{x \in \sphere^{d-1} \mid \frac{a_i}{\norm{a_i}}^Tx=\frac{b_i}{\norm{a_i}}}.
\end{split}
\end{equation}

Combining~\eqref{optCasei} and~\eqref{optCaseii} yields that
\begin{equation} \label{eqXi}
\opt_i = \begin{cases}
           \br{\frac{a_{i}}{\norm{a_{i}}}}, & \mbox{if } \norm{a_{i}} \leq b_i \\
           \br{x\in\sphere^{d-1} \mid \frac{a_i}{\norm{a_i}}^Tx=\frac{b_i}{\norm{a_i}}}, & \mbox{if } \norm{a_{i}} > b_i.
         \end{cases}.
\end{equation}
Therefore,
\begin{equation} \label{eqSizeXi}
|\opt_i| = \begin{cases}
           1, & \mbox{if } \norm{a_{i}} \leq b_i \\
           \infty, & \mbox{if } \norm{a_{i}} > b_i
         \end{cases}.
\end{equation}

Substituting $A,b$ and $x^*$ in Lemma~\ref{corAxb4Approx} yields that there is an index $j \in [n]$ and a corresponding unit vector $x_{j} \in \argmin_{x\in \sphere^{d-1}} |a_j^Tx-b_j| = \opt_{j}$ such that for every $i\in [n]$
\begin{equation} \label{eqApprox4}
|a_i^Tx_{j}-b_i| \leq 4\cdot |a_i^Tx^*-b_i|.
\end{equation}
We continue with the following case analysis: \textbf{(a):} $\norm{a_{j}} \leq b_{j}$ and \textbf{(b):} $\norm{a_{j}} > b_{j}$.

\textbf{Case (a):} $\norm{a_{j}} \leq b_{j}$. Similarly to~\eqref{eqXi}, we have that
\begin{equation} \label{eqXj}
\opt_j = \begin{cases}
           \br{\frac{a_{j}}{\norm{a_{j}}}}, & \mbox{if } \norm{a_{j}} \leq b_j \\
           \br{x\in\sphere^{d-1} \mid \frac{a_j}{\norm{a_j}}^Tx=\frac{b_j}{\norm{a_j}}}, & \mbox{if } \norm{a_{j}} > b_j.
         \end{cases}.
\end{equation}
Combining the previous definition of $\opt_j$ with the assumption of Case (a) yields
\begin{equation} \label{eqOptJ}
|\opt_j| = \left|\br{\frac{a_{j}}{\norm{a_{j}}}}\right| = 1.
\end{equation}
By combining~\eqref{eqApprox4},~\eqref{eqOptJ} and the fact that $\frac{a_{j}}{\norm{a_{j}}} \in \opt_j = \opt(a_j,b_j)$, we obtain that Lemma~\ref{lemApproxD} holds for the case
that $\norm{a_{j}} \leq b_{j}$ and every $d\geq 3$, by letting $r=1$, $i_1 = j$ and $x' = \frac{a_{j}}{\norm{a_{j}}}$.

\textbf{Case (b):} $\norm{a_{j}} > b_{j}$. Put $i\in [n]$.
Assume without loss of generality that $\frac{a_{j}}{\norm{a_{j}}} = e_d = (0,\cdots,0,1)\in \REAL^d$. Otherwise, rotate the coordinates system.
Hence,
\begin{equation} \label{eqDefXj}\begin{split}
\opt_j & = \br{x\in \sphere^{d-1} \mid \frac{a_{j}^T}{\norm{a_{j}}}x = \frac{b_{j}}{\norm{a_{j}}}}\\
& = \br{x\in\sphere^{d-1} \mid e_d^Tx = \frac{b_{j}}{\norm{a_{j}}}}\\
& = \br{x = (x^1,\cdots,x^d)\in\sphere^{d-1} \mid x^d = \frac{b_{j}}{\norm{a_{j}}}},
\end{split}
\end{equation}
where the first derivation holds by substituting $i=j$ in~\eqref{eqXi}, and the second derivation holds by the assumption that $\frac{a_{j}}{\norm{a_{j}}} = e_d$.

For every $m \in [n]$, identify the entries of $a_m$ by $(a_m^1,\cdots,a_m^d)$, let
\[
b_m' = -\left(\frac{a_m^d \cdot b_{j}}{\norm{a_{j}}} - b_{m}\right),
\]
and
\[
a_m' = \sqrt{1-\left(\frac{b_{j}}{\norm{a_{j}}}\right)^2}\cdot (a_m^1,\cdots,a_m^{d-1})^T \in \REAL^{d-1}.
\]
For every $x = (x^1,\cdots,x^d) \in \opt_j$, let
\[
v(x) = \frac{(x^1,\cdots,x^{d-1})^T}{\norm{(x^1,\cdots,x^{d-1})}} \in \sphere^{d-2},
\]
and observe that $\norm{(x^1,\cdots,x^{d-1})} \neq 0$ since $\norm{x}=1$ and $|x^d| = \frac{|b_j|}{\norm{a_j}} < 1$.
For every $x\in \opt_j$ we have that
\begin{align}
& |a_i^Tx-b_i|\\ &= \left| (a_i^1,\cdots,a_i^{d-1})(x^1,\cdots,x^{d-1})^T + a_i^d\cdot x^d - b_{i} \right| \nonumber\\
& = \left| (a_i^1,\cdots,a_i^{d-1})(x^1,\cdots,x^{d-1})^T + a_i^d \cdot \frac{b_{j}}{\norm{a_{j}}} - b_{i}\right| \label{eqMain1}\\
& = \left| (a_i^1,\cdots,a_i^{d-1})(x^1,\cdots,x^{d-1})^T - b_{i}'\right| \label{eqMain2}\\
& = \left| \sqrt{1-\left(\frac{b_{j}}{\norm{a_{j}}}\right)^2}\cdot (a_i^1,\cdots,a_i^{d-1})\frac{(x^1,\cdots,x^{d-1})^T}{\sqrt{1-\left(\frac{b_{j}}{\norm{a_{j}}}\right)^2}} - b_i' \right| \nonumber\\
& = \left| \sqrt{1-\left(\frac{b_{j}}{\norm{a_{j}}}\right)^2}\cdot (a_i^1,\cdots,a_i^{d-1})\frac{(x^1,\cdots,x^{d-1})^T}{\norm{(x^1,\cdots,x^{d-1})}} - b_i' \right| \label{eqMain3}\\
& = \left| a_i'^T v(x) -b_i' \right| \label{eqMain4},
\end{align}
where~\eqref{eqMain1} is by~\eqref{eqDefXj}, \eqref{eqMain2} is by the definition of $b_i'$, \eqref{eqMain3} holds since $x$ is a unit vector, and~\eqref{eqMain4} is by the definition of $a_i'$.

We continue the proof of Case (b) with the following subcase analysis: \textbf{b(i)} $a_i'= \vec{0}$ for every $i\in [n]\setminus\br{j}$, and \textbf{b(ii)} $a_l'\neq \vec{0}(d-1)$ for some index $l\in [n]\setminus\br{j}$.

\textbf{Subcase b(i): }$a_i'= \vec{0}(d-1)$ for every $i\in [n]\setminus\br{j}$.
Put $x' \in \opt_j$ and $i\in [n]\setminus\br{j}$. In this case, for every $\hat{x}\in\sphere^{d-2}$ we have that
\begin{equation} \label{eqAllTheSame}
|a_i'\hat{x}-b_i'| = |b_i'|.
\end{equation}
We thus obtain that
\begin{equation} \label{eqalloptj}
|a_i^Tx'-b_i| = |a_i'^Tv(x')-b_i'| = |b_i'|,
\end{equation}
where the first equality is by~\eqref{eqMain4} and the last equality is by~\eqref{eqAllTheSame}.

Combining~\eqref{eqalloptj} and the fact that $x_j \in \opt_j$ yields
\[
|a_i^Tx'-b_i| = |a_i^Tx_j-b_i|.
\]
Combining the last equality with~\eqref{eqApprox4} yields
\[
|a_i^Tx'-b_i| \leq 4\cdot |a_i^Tx^*-b_i|.
\]

Combining that the last equality holds for every $x'\in \opt_j$ and $i\in [n]$, with the fact that $|\opt_j| = \infty$ by~\eqref{eqXj} proves Subcase b(i) of Lemma~\ref{lemApproxD} for $X = \opt_j$.

\textbf{Subcase b(ii): }$a_l' \neq \vec{0}(d-1)$ for some $l\in [n]\setminus\br{j}$.
We have that
\begin{equation} \label{eqmain25}
|a_i'^Tv(x_{j})-b_i'| = |a_i^Tx_{j}-b_i| \leq 4 \cdot |a_i^Tx^*-b_i|,
\end{equation}
where the equality follows by substituting $x=x_j$ in~\eqref{eqMain4}, and the inequality holds by~\eqref{eqApprox4}.

Let $A' = (a_1' \mid \cdots \mid a_{j-1}' \mid a_{j+1}' \mid \cdots,a_n')^T\in \REAL^{(n-1)\times(d-1)}$, $b' = (b_1',\cdots,b_{j-1}',b_{j+1}',\cdots,b_n')^T \in \REAL^{n-1}$.
Applying the inductive assumption on $A',b'$ and $v(x_{j})$ yields that there exists a set $\br{{i_1},\cdots,{i_m}} \subseteq [n]\setminus\br{j}$ of
$m \in [d-2]$ indices such that for $\opt' = \opt((a_{i_1}'\mid\cdots\mid a_{i_m}'),(b_{i_1}',\cdots,b_{i_m}'))$ and every $i\in [n]\setminus\br{j}$ there is $\hat{x} \in \opt'$ that satisfies
\begin{equation}\label{item_i}
|a_i'\hat{x}-b_i'| \leq 4^{d-2}\cdot |a_i'v(x_{j})-b_i'|.
\end{equation}
Moreover, \eqref{item_i} holds for every $\hat{x}\in X'$ if $|\opt'|=\infty$.

Let
\begin{equation}\label{optDef1}
\opt = \opt((a_j\mid a_{i_1}\mid \cdots\mid a_{i_m}),(b_j,b_{i_1},\cdots,b_{i_m})).
\end{equation}

We continue to prove Lemma~\ref{lemApproxD} for Subcase b(ii) using another case analysis. \textbf{Subcase b(ii,1):} $|\opt'| \in O(1)$ and there is $\hat{x} \in \opt'$ that satisfies~\eqref{item_i}, and \textbf{Subcase b(ii,2):} $|\opt'| = \infty$ and every $\hat{x} \in \opt'$ satisfies~\eqref{item_i}.

\textbf{Subcase b(ii,1): }$|\opt'| \in O(1)$ and there is $\hat{x} \in \opt'$ that satisfies~\eqref{item_i}. We prove there is $x'\in \opt$ that satisfies~\eqref{mainToProve_proof} and that $|\opt| \in O(1)$.

Let
\[
x' = \left(\sqrt{1-\left(\frac{b_{j}}{\norm{a_{j}}}\right)^2}\cdot \hat{x}^T \mid \frac{b_{j}}{\norm{a_{j}}}\right)^T \in \REAL^d.
\]
Observe that
\begin{equation} \label{x2tox22}
v(x') = \hat{x} \in \opt',
\end{equation}
by the definitions of $x'$ and $\hat{x}$ respectively.
Combining the definition of $x'$ and~\eqref{eqDefXj} yields that $x' \in \opt_j \subseteq \sphere^{d-1}$. 
Therefore,
\begin{align}
|a_i^Tx'-b_i| & = |a_i'^Tv(x')-b_i'| \label{eqmain21}\\
& = |a_i'^T\hat{x}-b_i'| \label{eqmain22}\\
& \leq 4^{d-2}\cdot |a_i'^Tv(x_{j})-b_i'| \label{eqmain23}\\
& \leq 4^{d-1} \cdot |a_i^Tx^*-b_i|. \label{eqmainFinal}
\end{align}
where~\eqref{eqmain21} holds by substituting $x=x'$ in~\eqref{eqMain4}, \eqref{eqmain22} is by~\eqref{x2tox22}, \eqref{eqmain23} is by~\eqref{item_i},  and~\eqref{eqmainFinal} holds by~\eqref{eqmain25}.

We now present the following observation that will be used afterwards in the proof.
\begin{observation} \label{reductionObs}
\[
\opt = \br{x' \in \opt_j \mid v(x') \in \opt'}.
\]
\end{observation}
\begin{proof}
\begin{align}
\opt&=\opt((a_j\mid a_{i_1}\mid \cdots\mid a_{i_m}),(b_j,b_{i_1},\cdots,b_{i_m}))\label{proofReductionOpt1}\\
\begin{split}
&=\begin{cases}
\displaystyle \argmin_{\substack{x'\in\sphere^{d-1}: \\a_j^Tx'=b_j}} |a_{i_m}^Tx'-b_{i_m}|, & \mbox{if } m=1 \\
\displaystyle \argmin_{\substack{x'\in\sphere^{d-1}: \\(a_j \mid a_{i_1} \mid \cdots\mid a_{i_{m-1}})^Tx'=\\(b_j, b_{i_1},\cdots,b_{i_{m-1}})^T}} |a_{i_m}^Tx'-b_{i_m}|, & \mbox{otherwise}
\end{cases}
\end{split}
\label{proofReductionOpt2}\\
\begin{split}
&=\begin{cases}
\displaystyle \argmin_{\substack{x'\in\sphere^{d-1}: \\a_j^Tx'=b_j}} |a_{i_m}^Tx'-b_{i_m}|, & \mbox{if } m=1 \\
\displaystyle \argmin_{\substack{x'\in\sphere^{d-1}: \\a_j^Tx'=b_j,\\
(a_{i_1} \mid \cdots\mid a_{i_{m-1}})^Tx'=\\(b_{i_1},\cdots,b_{i_{m-1}})^T}} |a_{i_m}^Tx'-b_{i_m}|, & \mbox{otherwise}
\end{cases}
\end{split}\label{proofReductionOpt3}\\
\begin{split}
&=\begin{cases}
\displaystyle \argmin_{\substack{x'\in\sphere^{d-1}: \\\frac{a_j^T}{\norm{a_j}}x'=\frac{b_j}{\norm{a_j}}}} |a_{i_m}^Tx'-b_{i_m}|, & \mbox{if } m=1 \\
\displaystyle \argmin_{\substack{x'\in\sphere^{d-1}: \\\frac{a_j^T}{\norm{a_j}}x'=\frac{b_j}{\norm{a_j}},\\
(a_{i_1} \mid \cdots\mid a_{i_{m-1}})^Tx'=\\(b_{i_1},\cdots,b_{i_{m-1}})^T}} |a_{i_m}^Tx'-b_{i_m}|, & \mbox{otherwise}
\end{cases}
\end{split}\label{proofReductionOpt4}\\
\begin{split}
&=\begin{cases}
\displaystyle \argmin_{\substack{x'\in\sphere^{d-1}: \\
x' \in \opt_j}} |a_{i_m}^Tx'-b_{i_m}|, & \mbox{if } m=1 \\
\displaystyle \argmin_{\substack{x'\in\sphere^{d-1}: \\
x' \in \opt_j,\\
(a_{i_1} \mid \cdots\mid a_{i_{m-1}})^Tx'=\\(b_{i_1},\cdots,b_{i_{m-1}})^T}} |a_{i_m}^Tx'-b_{i_m}|, & \mbox{otherwise}
\end{cases}
\end{split}\label{proofReductionOpt5}\\
\begin{split}
&=\begin{cases}
\displaystyle \argmin_{\substack{x'\in\sphere^{d-1}: \\
x' \in \opt_j}} |a_{i_m}'^Tv(x')-b_{i_m}'|, & \mbox{if } m=1 \\
\displaystyle \argmin_{\substack{x'\in\sphere^{d-1}: \\
x' \in \opt_j,\\
(a_{i_1}' \mid \cdots\mid a_{i_{m-1}}')^Tv(x')=\\(b_{i_1}',\cdots,b_{i_{m-1}}')^T}} |a_{i_m}'^Tv(x')-b_{i_m}'|, & \mbox{otherwise}
\end{cases}
\end{split}\label{proofReductionOpt6}\\
&=\displaystyle \br{x' \in \opt_j \mid v(x') \in \opt((a_{i_1}'\mid\cdots\mid a_{i_m}'),(b_{i_1}',\cdots,b_{i_m}'))}\label{proofReductionOpt7}\\
&=\displaystyle \br{x' \in \opt_j \mid v(x') \in \opt'}\label{proofReductionOpt8}
\end{align}
where~\eqref{proofReductionOpt1} holds by~\eqref{optDef1}, ~\eqref{proofReductionOpt2} holds by ~\eqref{defOptAxb} and since there are at least $2$ entries in $(b_j,b_{i_1},\cdots,b_{i_m})$, ~\eqref{proofReductionOpt3} holds since it's the same as \eqref{proofReductionOpt2} just written differently. \eqref{proofReductionOpt4} holds since $\norm{a_{j}} > b_{j} \geq 0$ and therefore there's no division by $0$. \eqref{proofReductionOpt5} holds by ~\eqref{eqDefXj}. \eqref{proofReductionOpt6} holds by~\eqref{eqMain4}. \eqref{proofReductionOpt7} holds by ~\eqref{defOptAxb} and \eqref{proofReductionOpt8} holds by the definition of $\opt'$.
\end{proof}

Combining the fact that $x' \in \opt_j$ and~\eqref{x2tox22} with Observation~\ref{reductionObs} yields that
\[
x' \in \opt.
\]
It also holds that
\begin{equation} \label{constOptSize}
\begin{split}
& |\opt| = |\br{x' \in \opt_j \mid v(x') \in \opt'}|\\
& = \left|\br{x'=(x^1,\cdots,x^d) \mid x^d=\frac{b_j}{\norm{a_j}} \text{ and } v(x') \in \opt'}\right|\\
& \leq |\opt'| \in O(1),
\end{split}
\end{equation}
where the last derivation is by the assumption of Subcase b(ii,1).

Hence, Lemma~\ref{lemApproxD} holds for Subcase b(ii,1) with $X = \opt$ and $x' \in X$.

\textbf{Subcase b(ii,2): }$|\opt'| = \infty$ and every $\hat{x} \in \opt'$ satisfies~\eqref{item_i}. We prove that $|\opt| = \infty$ and that every $x'\in\opt$ satisfies~\eqref{mainToProve_proof}.

Similarly to~\eqref{constOptSize}, we have that $|\opt| \leq |\opt'| = \infty$.
For every $x' \in \opt$, by Observation~\ref{reductionObs} we have that $x' \in \opt_j$ and $v(x') \in \opt'$. Therefore, for every $x' \in \opt$ we have that
\[
\begin{split}
|a_i^Tx'-b_i| = |a_i'^Tv(x')-b_i'|& \leq 4^{d-2}\cdot |a_i'^Tv(x_{j})-b_i'|\\
&\leq 4^{d-1} \cdot |a_i^Tx^*-b_i|.
\end{split}
\]
where the first equality holds by substituting $x=x'$ in~\eqref{eqMain4}, the first inequality holds by substituting $\hat{x} = v(x') \in \opt'$ in~\eqref{item_i}, and the last inequality is by~\eqref{eqmain25}.

Hence, Lemma~\ref{lemApproxD} holds for Subcase b(ii,2) with $X = \opt$.

Combining Subcase b(ii,1) and Subcase b(ii,2) prove Lemma~\ref{lemApproxD} for Subcase (ii). Combining Subcase (i) and Subcase (ii) proves  Lemma~\ref{lemApproxD} holds for Case (b) for every $d\geq 3$.
Lemma~\ref{lemApproxD} now holds for every $d\geq 3$ by combining Case (a) and Case (b).
\end{proof}

\begin{theorem} [Theorem~\ref{Axb}] \label{Axb_proof}
Let $A = (a_1 \mid \cdots \mid a_n)^T \in \REAL^{n\times d}$ be a matrix of $n \geq d-1 \geq 1$ rows, and let $b = (b_1,\cdots,b_n)^T \in \REAL^n$.
Let $X \subseteq \sphere^{d-1}$ be the output of a call to \algnameXCandidates$(A,b)$; see Algorithm~\ref{calcX}. Then for every $x^* \in \sphere^{d-1}$ there is a unit vector $x' \in \sphere^{d-1}$ such that for every $i\in[n]$,
\[
|a_i^Tx'-b_i| \leq 4^{d-1}\cdot |a_i^Tx^*-b_i|.
\]
Moreover, the set $X$ can be computed in $n^{O(d)}$ time and its size is $|X| \in n^{O(d)}$.
\end{theorem}
\begin{proof}
We use the variables as in Algorithm~\ref{calcX} for a call to \algnameXCandidates$(A,b)$.

Put $x^* \in \sphere^{d-1}$ and $i\in [n]$.
In Lines~\ref{line:SignFor}--\ref{line:endSignFor} of Algorithm~\ref{calcX}, we define
\[
b_{i}' = |b_{i}|,
\quad\quad a_{i}' = \sign(b_{i})\cdot a_{i}.
\]
Observe that for every $a\in\REAL^d$, $\hat{b}\in \REAL$ and every $x\in \REAL^d$
\begin{equation} \label{invertSign}
|a^Tx-\hat{b}| = |-a^Tx-(-\hat{b})|.
\end{equation}

Substituting $A$ with $A'=\br{a_1',\cdots,a_n'}$, and $b$ with $b'=\br{b_1',\cdots,b_n'}$ and plugging $x^*$ in Lemma~\ref{lemApproxD} yields that there is a set $\br{j_1,\cdots,j_r} \subseteq [n]$ of $m \in [d-1]$ indices, such that for $X' = \opt((a_{j_1}'\mid\cdots\mid a_{j_m}'), (b_{j_1}',\cdots,b_{j_m}'))$ and every $i\in [n]$, there is $\hat{x} \in X'$ that satisfies
\begin{equation} \label{eqx2inopt}
|a_i'^T\hat{x}-b_i'| \leq 4^{d-1}\cdot |a_i'^Tx^*-b_i'|.
\end{equation}
Moreover, if $|X'| = \infty$ then every $\hat{x} \in X'$ satisfies~\eqref{eqx2inopt}.

Let $x'$ be $\hat{x}$ if $|X'| \in O(1)$, and $x' \in X'$ be an arbitrary element if $|X'| = \infty$. Hence, $x'$ satisfies~\eqref{eqx2inopt}.
We now have that
\begin{equation} \label{mainPropToSatisfy}
\begin{split}
|a_i^Tx'-b_i| = |a_i'^Tx'-b_i'| & \leq 4^{d-1} \cdot |a_i'^Tx^*-b_i'|\\
&= 4^{d-1} \cdot |a_i^Tx^*-b_i|,
\end{split}
\end{equation}
where the first and last equalities hold by combining the definitions of $a_i'$ and $b_i'$ with~\eqref{invertSign}, and the inequality holds since $x'$ satisfies~\eqref{eqx2inopt}.

It is left to prove that $x'$ is in the output set $X$.
In Lines~\ref{line:ES}--\ref{line:addToX} of Algorithm~\ref{calcX}, we iterate over every $r \in [d-1]$ and every subset $\br{i_1,\cdots,i_r} \in [n]$ of $r$ indices and compute the set $S = \opt((a_{i_1}'\mid\cdots\mid a_{i_r}'),(b_{i_1}',\cdots,b_{i_r}'))$ using Algorithm~\ref{calcOpt}.

Therefore, when $r=m$, and $i_1=j_1,\cdots,i_r = j_r$, the call to Algorithm~\ref{calcOpt} in Line~\ref{line:compOpt} is guaranteed to compute a set $S$ that satisfies
\[
S =\begin{cases}
     X', & \mbox{if } |X'|\in O(1) \\
     \hat{x} \in X', & \mbox{otherwise}.
   \end{cases}.
\]
We then add $S$ to the output set $X$. By the definition of $x'$ and the output guarantee of Algorithm~\ref{calcOpt}, $x' \in S$. Hence, the output set $X$ is guaranteed to contain an element that satisfies~\eqref{mainPropToSatisfy}.

The running time of Algorithm~\ref{calcX} is dominated by the nested $d\cdot \binom{n}{d} \in n^{O(d)}$ ''for`` loops, each iteration taking $O(d)$ time, for a total of $n^{O(d)}$ running time.
In each iteration, we add $O(1)$ unit vectors to $X$. Hence, $|X| \in n^{O(d)}$.
\end{proof}

\begin{theorem} [Theorem~\ref{costAxb}] \label{costAxb_proof}
Let $A = (a_1 \mid \cdots \mid a_n)^T \in \REAL^{n\times d}$ be a matrix of $n \geq d-1 \geq 1$ rows, and let $b = (b_1,\cdots,b_n)^T \in \REAL^n$. Let $\cost, s, r$ be as defined in Definition~\ref{def:cost} for $Y = \br{(a_i,b_i)\mid i\in [n]}$ and $D((a,\hat{b}),x) = |a^Tx-\hat{b}|$ for every $a\in\REAL^d$, $\hat{b} \in \REAL$ and $x\in \REAL^d$. Then in $n^{O(d)}$ time we can compute a unit vector $x'\in\sphere^{d-1}$ such that
\[
\cost(Y,x') \leq 4^{(d-1)rs} \cdot \min_{x\in\sphere^{d-1}} \cost(Y,x).
\]
\end{theorem}
\begin{proof}
Let $x^* \in \argmin_{x\in\sphere^{d-1}} \cost(Y,x)$.
Plugging $A,b$ and $x^*$ in Theorem~\ref{Axb} yields that in $n^{O(d)}$ time we can compute a unit vector $x'\in\sphere^{d-1}$ such that for every $i\in [n]$
\[
\begin{split}
D((a_i,b_i),x')=|a_i^Tx'-b_i| & \leq 4^{d-1}\cdot |a_i^Tx^*-b_i|\\ &= 4^{d-1}\cdot D((a_i,b_i),x^*).
\end{split}
\]
Theorem~\ref{costAxb} now holds by plugging $Y, D, \cost, r,s$, $q'=x'$, $q^*=x^*$ and $c=4^{d-1}$ in Observation~\ref{obs:distToCost} as
\[
\begin{split}
\cost(Y,x') &\leq 4^{(d-1)rs} \cost(Y,x^*)\\ &= 4^{(d-1)rs} \min_{x\in\sphere^{d-1}} \cost(Y,x).
\end{split}
\]
\end{proof}

\section{Handling Unknown Matching}

\begin{theorem} [Theorem~\ref{costAxb_noMatch}] \label{costAxb_noMatch_proof}
Let $A = (a_1 \mid \cdots \mid a_n)^T \in \REAL^{n\times d}$ be a matrix containing $n \geq d-1 \geq 1$ rows, let $b = (b_1,\cdots,b_n)^T \in \REAL^n$, and let $D((a,\hat{b}),x) = |a^Tx-\hat{b}|$ for every $a\in\REAL^d$, $\hat{b} \in \REAL$ and $x\in \REAL^d$. Let $\cost, r$ be as defined in Definition~\ref{def:cost} for $Y = \br{(a_i,b_i)\mid i\in [n]}$, $D$ and $f(v) = \norm{v}_1$. Let $(\tilde{x},\tilde{\M})$ be a pair of unit vector and permutation (matching function) which is the output of a call to $\matchingalgname(A,b,\cost)$; see Algorithm~\ref{Alg_NoMatching}. Then it holds that
\[
\cost(Y_{\tilde{\M}},\tilde{x}) \leq 4^{(d-1)r} \cdot \min_{x,\M} \cost(Y_{\M},x),
\]
where the minimum is over every unit vector $x\in \sphere^{d-1}$ and $\M\in\perms(n)$. Moreover, $(\tilde{x},\tilde{\M})$ is computed in $n^{O(d)}$ time.
\end{theorem}
\begin{proof}
Put $(x^*,\M^*) \in \argmin_{x,\M}\cost(Y_{\M},x)$, where the minimum is over every unit vector $x\in \sphere^{d-1}$ and permutation $\M \in \perms(n)$, and put $i\in [n]$.

In Lines~\ref{line:SignFor2}-\ref{line:endSignFor2} of Algorithm~\ref{Alg_NoMatching}, we define
\[
b_{i}' = |b_{i}|,
\quad\quad a_{i}' = \sign(b_{i})\cdot a_{i}.
\]
For every $a\in\REAL^d$, $\hat{b}\in \REAL$ and every $x\in \sphere^{d-1}$, we have
\begin{equation} \label{invertSign2}
|a^Tx-\hat{b}| = |-a^Tx-(-\hat{b})|.
\end{equation}

Substituting $A$ with $A'=\br{a_1',\cdots,a_n'}$, $b$ with $b'=\br{b_{\M^*(1)}',\cdots,b_{\M^*(n)}'}$ and plugging $x^*$ in Lemma~\ref{lemApproxD} yields that there is a set $\br{j_1,\cdots,j_m} \subseteq [n]$ of $m \in [d-1]$ indices, such that for $X' = \opt((a_{j_1}'\mid\cdots\mid a_{j_m}'), (b_{\M^*(j_1)}',\cdots,b_{\M^*(j_m)}'))$ and every $i\in [n]$, there is $\hat{x} \in X'$ that satisfies
\begin{equation} \label{eqx2inopt11}
|a_i'^T\hat{x}-b_{M^*(i)}'| \leq 4^{d-1} \cdot |a_i'^Tx^*-b_{M^*(i)}'|.
\end{equation}
Moreover, if $|X'| = \infty$ then every $\hat{x} \in X'$ satisfies~\eqref{eqx2inopt11}.

Let $x'$ be $\hat{x}$ if $|X'| \in O(1)$, and $x' \in X'$ be an arbitrary element if $|X'| = \infty$. Hence, $x'$ satisfies~\eqref{eqx2inopt11}.
We now have
\begin{equation} \label{mainPairwiseApprox}
\begin{split}
|a_i^Tx'-b_{\M^*(i)}| & = |a_i'^Tx'-b_{\M^*(i)}'|\\
& \leq 4^{d-1} \cdot |a_i'^Tx^*-b_{\M^*(i)}'|\\
&= 4^{d-1} \cdot |a_i^Tx^*-b_{\M^*(i)}|.
\end{split}
\end{equation}
where the first and last equalities hold by combining the definitions of $a_i'$ and $b_i'$ with~\eqref{invertSign2}, and the inequality holds since $x'$ satisfies~\eqref{eqx2inopt11}.

Since~\eqref{mainPairwiseApprox} holds for every $i\in [n]$, we can substitute $Y = Y_{\M^*} = \br{(a_i,b_{\M^*(i)})}_{i=1}^n$, $q' = x'$, $q^*=x^*$ and $c=4^{d-1}$ in Observation~\ref{obs:distToCost} to obtain that
\begin{equation}\label{eqCostMStart}
\cost(Y_{\M^*},x') \leq 4^{r(d-1)}\cdot \cost(Y_{\M^*},x^*).
\end{equation}

It is left to prove that $x'$ is in the output set $X$.
In Lines~\ref{line:ES2}-\ref{line:addToX2} of Algorithm~\ref{Alg_NoMatching}, we iterate over every $r \in [d-1]$ and every subsets $\br{i_1,\cdots,i_r} \subseteq [n]$ and $\br{\ell_1,\cdots,\ell_r} \subseteq [n]$ of $r$ indices, we compute the set $S = \opt((a_{i_1}'\mid\cdots\mid a_{i_r}'),(b_{\ell_1}',\cdots,b_{\ell_r}'))$ using Algorithm~\ref{calcOpt}, we then add $S$ to $X$.

Therefore, when $r=m$, $i_1=j_1,\cdots,i_r = j_r$ and $\ell_1 = \M^*(j_1),\cdots,\ell_r = \M^*(j_r)$, the call to Algorithm~\ref{calcOpt} in Line~\ref{line:compOpt2} is guaranteed to compute a set $S$ that satisfies
\[
S =\begin{cases}
     X', & \mbox{if } |X'|\in O(1) \\
     \hat{x} \in X', & \mbox{otherwise}.
   \end{cases}.
\]
We then add $S$ to the output set $X$. By the definition of $x'$ and the output guarantee of Algorithm~\ref{calcOpt}, $x' \in S$. Hence, the output set $X$ is guaranteed to contain an element that satisfies~\eqref{mainPairwiseApprox}.

Afterwards, in Line~\ref{line:compOptMatch}, we compute the optimal permutation $\hat{\M}(Y,x,\cost)$ for every unit vector $x\in X$.
To compute the optimal matching one could compute the pairwise cost matrix $Z \in \REAL^{n\times n}$, where the entry $(i,j)$ contains $\cost((a_i,b_j),x)$ (i.e., the cost of pairing $a_i$ with $b_j$ when using $x$), and then apply the Hungarian Method~\cite{kuhn1955hungarian}.

Since there is $x' \in X$ that satisfies~\eqref{mainPairwiseApprox}, we get that $(x',\hat{\M}(Y,x',\cost)) \in S$. Let $M' = \hat{\M}(Y,x,\cost)$.
By Definition~\ref{def:optimalMatching} of optimal matching,
\begin{equation} \label{eqApproxMatch}
\cost(Y_{M'},x') \leq \min_{\M \in \perms(n)}\cost(Y_{\M},x') \leq \cost(Y_{M^*},x').
\end{equation}
In Line~\ref{line:compBestPair}, we pick the pair $(\tilde{x},\tilde{\M}) \in \argmin_{(x,\M)\in S} \cost\left(Y_{\M},x\right)$. Therefore, the main claim of Theorem~\ref{costAxb_noMatch} holds as
\[
\begin{split}
\cost(Y_{\tilde{\M}},\tilde{x}) \leq \cost(Y_{\M'},x')& \leq \cost(Y_{M^*},x')\\
&\leq 4^{r(d-1)}\cdot \cost(Y_{\M^*},x^*),
\end{split}
\]
where the first inequality is by the definition of $(\tilde{x},\tilde{\M})$, the second inequality is by~\eqref{eqApproxMatch} and the last inequality is by~\eqref{eqCostMStart}.

Since the call to Algorithm~\ref{calcOpt} in Line~\ref{line:compOpt2} returns at most $2$ unit vectors, the size of the set $X$ is proportional to the number of iterations of the nested for loops, which is $d\cdot \binom{n}{d}^2 \in n^{O(d)}$.
The running time of Algorithm~\ref{Alg_NoMatching} is dominated by Line~\ref{line:compOptMatch}, which runs an optimal matching algorithm for every unit vector in $X$.
The Hungarian method for optimal matching takes $O(n^3)$ time. Hence, the total running time is $n^{O(d)}$.
\end{proof}

\section{Coreset for Linear Regression}

\begin{definition} [Definition 14 in~\cite{varadarajan2012sensitivity}]
Let $M$ be an $n\times m$ matrix of rank $\rho$. Let $z\in [1,\infty)$, and $\alpha,\beta \geq 1$. An $n\times \rho$ matrix $U$ is an $(\alpha,\beta,z)$-conditioned basis
for $M$ if the column vectors of $U$ span the column space of $M$, and additionally $U$ satisfies that: (1) $\sum{i,j} |a_{i,j}|^z \leq \alpha^z$, (2) for all $u\in \REAL^{\rho}$,
$\norm{u}_{z'} \leq \beta \norm{Uu}_z$, where $\norm{\cdot}_{z'}$ is the dual norm for $\norm{\cdot}_z$ (i.e. $1/z$ + $1/{z'}=1$).
\end{definition}

\begin{lemma} [Lemma 15 in~\cite{varadarajan2012sensitivity}] \label{lemSendBound1}
Let $M = (m_1\mid\cdots\mid m_n)^T$ be an $n\times m$ matrix of rank $\rho$. Let $z\in [1,\infty)$. Let $U = (u_1\mid\cdots\mid u_n)^T$ be an $(\alpha,\beta,z)$-conditioned basis for $M$. For every vector $u\in \REAL^m$, the following inequality
holds: $|m_i^Tu|^z \leq (\norm{u_i}_z^z \cdot \beta^z)\norm{Mu}_z^z$.
\end{lemma}

The following is a restatement of Lemma 16 in~\cite{varadarajan2012sensitivity}
\begin{lemma} [(total sensitivity for fitting a hyperplane~\cite{varadarajan2012sensitivity})] \label{lemTotalSensBound}
Let $P = \br{p_1,\cdots,p_n} \subseteq \REAL^d$ be a set of $n \geq d$ points. Let $\F = \br{(x,b) \in \REAL^{d}\times\REAL \mid \norm{x}=1}$ be the set of all hyperplanes in $\REAL^d$ and let $z \in [1,\infty)$.
Let $s:P \to [0,\infty)$ where $s(p_i) := \sup_{(x,b)\in \F} \frac{|p_i^Tx-b|^z}{\sum_{j\in [n]}|p_j^Tx-b|^z}$ is the sensitivity of a point $p_i \in P$.
Then the total sensitivity $\sum_{i\in [n]} s(p_i)$ of $P$ is $O(d^{1+z/2})$ if $z\in [1,2)$, $O(d)$ for $z=2$ and $O(d^z)$ for $z>2$.
\end{lemma}

\begin{corollary} \label{lem:sensBound}
Let $P = \br{p_1,\cdots,p_n} \subseteq \REAL^d$ be a set of $n \geq d$ points and let $z\in [1,\infty)$. Then a function $s:P \to [0,\infty)$ can be computed in $O(nd^5 \log{n})$
time such that for every $i\in [n]$,
\[
s(p_i) \geq \sup_{(x,b) \in \F} \frac{|p_i^T x-b|^z}{\sum_{j \in [n]}|p_j^Tx-b|^z},
\]
where the sup is over every $(x,b) \in \F$ such that $\sum_{j \in [n]}|p_j^Tx-b|^z \neq 0$.
\end{corollary}
\begin{proof}
For every $i\in [n]$, let $M_i = (p_i^T \mid 1)^T$, and let $M = (m_1 \mid \cdots \mid m_n)^T \in \REAL^{n\times (d+1)}$. Let $U = (u_1\mid\cdots\mid u_n)^T$ be a $(\alpha,\beta,z)$-conditioned basis for $M$.
Put $i\in [n]$. Then for every unit vector $x\in \REAL^d$ and $b \geq 0$
\begin{equation} \label{sensEq1}
|p_i^T x-b|^z = |(p_i^T \mid 1) (x^T\mid b)^T|^z = |m_i^T (x^T\mid b)^T|^z,
\end{equation}
and
\begin{equation} \label{sensEq2}
\begin{split}
\sum_{j \in [n]}|p_j^Tx-b|^z & = \sum_{j \in [n]}|(p_j^T \mid 1) (x^T\mid b)^T|^z\\
& = \sum_{j \in [n]}|m_j^T (x^T\mid b)^T|^z\\
&= \norm{M (x^T\mid b)^T}_z^z.
\end{split}
\end{equation}

It thus holds that
\[
\begin{split}
\sup_{(x,b) \in \F} \frac{|p_i^T x-b|^z}{\sum_{j \in [n]}|p_j^Tx-b|^z}& = \sup_{(x,b) \in \F} \frac{|m_i^T (x^T\mid b)^T|^z}{\norm{M (x^T\mid b)^T}_z^z}\\
& \leq \norm{u_i}_z^z \cdot \beta^z,
\end{split}
\]
where the first equality is by~\eqref{sensEq1} and~\eqref{sensEq2}, and the inequality holds by plugging $M$, $U$ and $u = (x^T\mid b)^T$ in Lemma~\ref{lemSendBound1}.
Hence, by letting $s(p_i) = \norm{u_i}_z^z \cdot \beta^z$, we get that
\[
s(p_i) \geq \sup_{(x,b) \in \F} \frac{|p_i^T x-b|^z}{\sum_{j \in [n]}|p_j^Tx-b|^z}.
\]

The time it takes to compute the function $s$ is dominated by the time to compute the matrix $U$. By Theorem 3.1 in~\cite{dasgupta2009sampling},
it takes $O(nd^5\log{n})$ time to compute the $(\alpha,\beta,z)$-conditioned basis matrix $U$ for $M$ since $M$ is an $n \times (d+1)$ matrix of rank at most $d+1$.
Hence, it takes $O(nd^5\log{n})$ time to compute the function $s$.
\end{proof}

\begin{definition} [\textbf{Definition 4.2 in~\cite{braverman2016new}}] \label{def:querySpace}
Let $P$ be a finite set, and let $w:P\to [0,\infty)$. Let $Q$ be a function that maps every set $S\subseteq P$ to a corresponding set $Q(S)$, such that $Q(T) \subseteq Q(S)$ for every $T\subseteq S$. Let $f:P\times Q(P) \to \REAL$ be a cost function. The tuple $(P,w,Q,f)$ is called a \emph{query space}.
\end{definition}

\begin{definition} [\textbf{VC-dimension}] \label{def:VC}
For a query space $(P,w,Q,f)$, $S\subseteq P$, $q\in Q(S)$ and $r\in [0,\infty)$ we define $\range(q,r) = \br{p\in P \mid w(p)\cdot f(p,q) \leq r}$.
The VC-dimension of $(P,w,Q,f)$ is the smallest integer $d_{VC}$ such that for every $S\subseteq P$ we have
\[
\left| \br{\range(q,r) \mid q\in Q(S), r\in [0,\infty)} \right| \leq |S|^{d_{VC}}.
\]
\end{definition}

\begin{theorem} [\textbf{Theorem 5.5 in~\cite{braverman2016new}}] \label{theorem:howToCoreset}
Let $(P,w,Q,f)$ be a query space; see Definition~\ref{def:querySpace}. Let $s:P\to [0,\infty)$ such that
\[
\sup_{q} \frac{w(p)f(p,q)}{\sum_{p\in P} w(p)f(p,q)} \leq s(p),
\]
for every $p\in P$ and $q\in Q(P)$ such that the denominator is non-zero. Let $t = \sum_{p\in P} s(p)$ and Let $d_{VC}$ be the VC-dimension of query space $(P,w,Q,f)$; See Definition~\ref{def:VC}. Let $c \geq 1$ be a sufficiently large constant and let $\varepsilon, \delta \in (0,1)$. Let $S$ be a random sample of
\[
|S| \geq \frac{ct}{\varepsilon^2}\left(d_{VC}\log{t}+\log{\frac{1}{\delta}}\right)
\]
points from $P$, such that $p$ is sampled with probability $s(p)/t$ for every $p\in P$. Let $u(p) = \frac{t\cdot w(p)}{s(p)|S|}$ for every $p\in S$. Then, with probability at least $1-\delta$, for every $q\in Q$ it holds that
\[
\begin{split}
(1-\varepsilon)\sum_{p\in P} w(p)\cdot f(p,q) & \leq \sum_{p\in S} u(p)\cdot f(p,q)\\ &\leq (1+\varepsilon)\sum_{p\in P} w(p)\cdot f(p,q).
\end{split}
\]
\end{theorem}

\begin{theorem} [Theorem~\ref{theorem:coreset}]\label{theorem:coreset_proof}
Let $d \geq 2$ be a constant integer. Let $A = (a_1 \mid \cdots \mid a_n)^T \in \REAL^{n\times d}$ be a matrix containing $n \geq d-1$ points in its rows, let $b = (b_1,\cdots,b_n)^T \in \REAL^n$, and let $w = (w_1,\cdots,w_n) \in [0,\infty)^n$.
Let $\varepsilon, \delta \in (0,1)$ and let $z\in [1,\infty)$.
Then in $O(n \log{n})$ time we can compute a weights vector $u = (u_1,\cdots,u_n)\in[0,\infty)^n$ that satisfies the following pair of properties.
\renewcommand{\labelenumi}{(\roman{enumi})}
\begin{enumerate}
\item With probability at least $1-\delta$, for every $x\in \sphere^{d-1}$ it holds that
\[
\begin{split}
(1-\varepsilon) \cdot \sum_{i\in [n]} w_i \cdot & |a_i^Tx-b_i|^z  \leq \sum_{i\in [n]}u_i\cdot |a_i^Tx-b_i|^z\\ &\leq (1+\varepsilon) \cdot \sum_{i\in [n]} w_i \cdot |a_i^Tx-b_i|^z.
\end{split}
\]
\item The weights vector $u$ has $O\left(\frac{\log{\frac{1}{\delta}}}{\varepsilon^2}\right)$ non-zero entries.
\end{enumerate}
\end{theorem}
\begin{proof}
Put $x\in \sphere^{d-1}$ and $i\in [n]$.
Let $W \in \REAL^{n\times n}$ be a diagonal matri whose diagonal elements are $(w_1^{1/z},\cdots,w_n^{1/z})$.
Let $A' = W A \in \REAL^{n\times d}$, $b' = Wb \in \REAL^{n}$ and $P = (p_1\mid\cdots\mid p_n)^T= [A' \mid b'] \in \REAL^{n\times (d+1)}$. For every unit vector $x\in \REAL^d$ it holds that
\[
\sum_{i\in [n]} w_i \cdot |a_i^Tx-b_i|^z = \norm{A'x-b'}_z^z = \norm{P (x^T \mid 1)^T}_z^z.
\]

Plugging $P$ in Corollary~\ref{lem:sensBound} yields that in $O(nd^5\log{n})$ time we can compute a function $s:P\to [0,\infty)$ such that
\begin{equation}\label{eqCompSFunc}
s(p_i) \geq \sup_{(y,b) \in \sphere^{d}\times\REAL} \frac{|p_i^T y-b|^z}{\norm{P y-b}_z^z}.
\end{equation}

Therefore, we have that
\begin{align}
\sup_{x \in \sphere^{d-1}}& \frac{w_i \cdot |a_i^Tx-b_i|^z}{\sum_{j\in [n]} w_j \cdot |a_j^Tx-b_j|^z}
= \sup_{x \in \sphere^{d-1}} \frac{|p_i^T (x^T \mid 1)^T|^z}{\norm{P (x^T \mid 1)^T}_z^z} \label{eqMainSens1}\\
& \leq \sup_{y \in \sphere^{d}} \frac{\sqrt{2} |p_i^T y|^z}{\sqrt{2} \norm{P y}_z^z} \label{eqMainSens2}\\
& \leq \sup_{(y,b) \in \sphere^{d}\times\REAL} \frac{|p_i^T y-b|^z}{\norm{P y-b}_z^z} \label{eqMainSens3}\\
& \leq s(p_i).
\end{align}
where~\eqref{eqMainSens1} holds by the definition of $P$, \eqref{eqMainSens2} holds since $\br{\frac{(x^T \mid 1)^T}{\sqrt{2}} \mid \norm{x} = 1} \subseteq \sphere^{d}$,
\eqref{eqMainSens3} holds since for every $y \in \sphere^{d}$ there is $(a,b) \in \sphere^{d}\times\REAL$ such that $a=y$, and~\eqref{eqMainSens3} is by~\eqref{eqCompSFunc}.

The VC-dimension of the corresponding query space is bounded by $d_{VC} \in O(d)$ by~\cite{anthony2009neural}; See Definition~\ref{def:VC}.

Let $t = \sum_{i\in [n]} s(p_i)$ and $Y = \br{(a_i,b_i)}_{i=1}^n$.
By Lemma~\ref{lemTotalSensBound} we have that $t \in O(d^{1+z/2})$ if $z\in [1,2)$, $O(d)$ if $z=2$ and $O(d^z)$ if $z>2$.

Let $Z \subseteq Y$ be a random sample of
\[
|Z| \in O\left(\frac{t}{\varepsilon^2}\left(d\log{t}+\log{\frac{1}{\delta}}\right)\right)
=\frac{d^{O(1)}}{\varepsilon^2}\log{\frac{1}{\delta}}
\]
pairs from $Y$, where $(a_i,b_i) \in Y$ is sampled with probability $s(p_i)/t$ and let $u = (u_1,\cdots,u_n)$ where
\[
u_i = \begin{cases} \frac{t\cdot w_i}{s(p_i)|Z|}, & \text{ if } (a_i,b_i)\in Z\\
0, & \text{ otherwise} \end{cases}.
\]
By substituting $P=Y$, $Q(\cdot)\equiv\sphere^{d-1}$, $f\left((a,b),x\right)=w_i\cdot |a^Tx-b|^z$ for every $(a,b)\in Y$ and $x\in \sphere^{d-1}$, $t \in d^{O(1)}$ and $d_{VC}=O(d)$, in Theorem~\ref{theorem:howToCoreset}, Property (i) of Theorem~\ref{theorem:coreset} holds as
\[
\begin{split}
(1-\varepsilon) \cdot \sum_{i\in [n]}& w_i \cdot |a_i^Tx-b_i|^z \leq \sum_{i\in [n]}u_i\cdot |a_i^Tx-b_i|^z\\
& \leq (1+\varepsilon) \cdot \sum_{i\in [n]} w_i \cdot |a_i^Tx-b_i|^z.
\end{split}
\]

Furthermore, Property (i) of Theorem~\ref{theorem:coreset} holds since the number of non-zero entries of $u$ is equal to $|Z| \in \frac{d^{O(1)}}{\varepsilon^2}\log{\frac{1}{\delta}} = O\left(\frac{\log{\frac{1}{\delta}}}{\varepsilon^2}\right)$.

The time needed to compute $u$ is bounded by the computation time of $s$, which is bounded by $O(n\log{n})$ since $d$ is a constant.
\end{proof}

\end{document}